\documentclass[lettersize,journal]{IEEEtran}
\usepackage{amsmath,amsfonts}
\usepackage{algorithmic}
\usepackage[ruled,lined]{algorithm2e}
\usepackage{array}
\usepackage[caption=false,font=normalsize,labelfont=sf,textfont=sf]{subfig}
\usepackage{textcomp}
\usepackage{stfloats}
\usepackage{url}
\usepackage{verbatim}
\usepackage{graphicx}
\usepackage{cite}
\usepackage{pifont}
\usepackage{makecell}
\usepackage{bm}
\usepackage{bbm}
\usepackage[english]{babel}
\usepackage{amsthm}
\usepackage{color}

\DeclareMathOperator*{\argmax}{arg\,max}

\newtheorem{theorem}{Theorem}

\newtheorem{assumption}{Assumption}
\newtheorem{lemma}{Lemma}
\newtheorem{corollary}{Corollary}

\begin{document}

\title{Bayesian Optimization with Formal Safety Guarantees via Online Conformal Prediction}

\author{Yunchuan Zhang, \IEEEmembership{Graduate Student Member,~IEEE}, Sangwoo Park, \IEEEmembership{Member,~IEEE}, \\and~Osvaldo Simeone, \IEEEmembership{Fellow,~IEEE}
\thanks{The authors are with the King’s Communications, Learning and Information Processing (KCLIP) lab within the Centre for Intelligent Information Processing Systems (CIIPS), Department of Engineering, King’s College London, London, WC2R 2LS, UK. (email:\{yunchuan.zhang, sangwoo.park, osvaldo.simeone\}@kcl.ac.uk). This work was supported by the European Research Council (ERC) under the
European Union’s Horizon 2020 Research and Innovation Programme (grant agreement No. 725732), by the European Union’s Horizon Europe project CENTRIC (101096379),  by an Open Fellowship of the EPSRC (EP/W024101/1), and by the EPSRC project (EP/X011852/1).}
}

\markboth{Accepted By IEEE Journal of Selected Topics in Signal Processing}%
{Shell \MakeLowercase{\textit{et al.}}: A Sample Article Using IEEEtran.cls for IEEE Journals}


\maketitle

\begin{abstract}
Black-box zero-th order optimization is a central primitive for applications in fields as diverse as finance, physics, and engineering. In a common formulation of this problem, a designer sequentially attempts  candidate solutions, receiving noisy feedback on the value of each attempt from the system. In this paper, we study scenarios in which feedback is also provided on the \emph{safety} of the attempted solution, and the optimizer is constrained to limit the number of unsafe solutions that are tried throughout the optimization process.  Focusing on methods based on Bayesian optimization (BO), prior art has introduced an optimization scheme -- referred to as \textsf{S{\footnotesize AFE}O{\footnotesize PT}} --  that is guaranteed not to select \emph{any} unsafe solution with a controllable probability over feedback noise as long as strict assumptions on the safety constraint function are met. In this paper, a novel BO-based approach is introduced that satisfies safety requirements irrespective of properties of the constraint function. This strong theoretical guarantee is obtained at the cost of allowing for an arbitrary, controllable but non-zero, rate of violation of the safety constraint. The proposed method, referred to as \textsf{S{\footnotesize AFE}-B{\footnotesize OCP}}, builds on online conformal prediction (CP) and is specialized to the cases in which feedback on the safety constraint is either noiseless or noisy.  Experimental results on synthetic and real-world data validate the advantages and flexibility of the proposed \textsf{S{\footnotesize AFE}-B{\footnotesize OCP}}. 
\end{abstract}


\section{Introduction}\label{sec: intro}
\subsection{Context and Scope}\label{ssec: context and scope}
\IEEEPARstart{P}{roblems} as diverse as stock portfolio optimization and asset management \cite{michaud2008efficient}, capacity allocation in energy systems \cite{wang2019capacity}, material discovery \cite{xu2021self}, calibration and optimization of quantum systems \cite{cortes2022quantbo}, and scheduling and optimization of wireless systems \cite{zhang2022scheduling,zhang2023metabo} can all be formulated as \emph{black-box zero-th order} optimizations. In such problems, the objective to be optimized can only be accessed on individual candidate solutions, and no further information is retrieved apart from the value of the objective. As illustrated in Fig. \ref{fig: intro flow}, in a common formulation of this problem, a designer sequentially attempts  candidate solutions, receiving noisy feedback on the value of each attempt from the system. In this paper, we study scenarios in which feedback is also provided on the \emph{safety} of the attempted solution, and the optimizer is constrained to limit the number of unsafe solutions that are tried throughout the optimization process \cite{sui2015safeopt, Felix2016safeopt, matteo2019goose,sui2018stagewise,  rothfuss2023meta}. 

\begin{figure}[t]

  \centering
  \centerline{\includegraphics[scale=0.37]{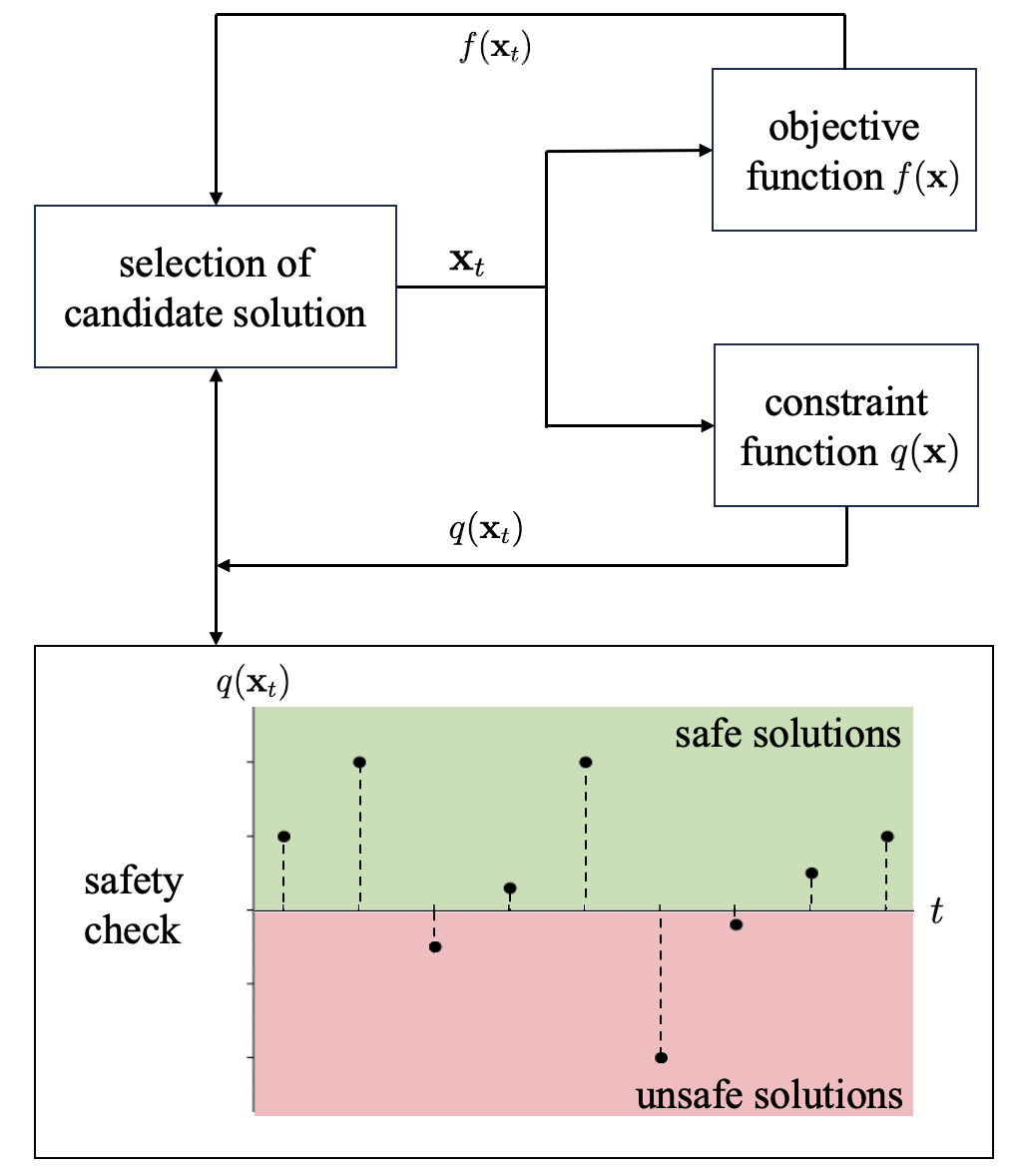}}
  \caption{This paper studies black-box zero-th order optimization with safety constraints. At each step $t=1,2,...$ of the sequential optimization process, the optimizer selects a candidate solution $\mathbf{x}_t$ and receives noisy feedback on the values of the objective function $f(\mathbf{x}_t)$ and of the constraint function $q(\mathbf{x}_t)$. Candidate solutions $\mathbf{x}_t$ yielding a negative value for the constraint function, $q(\mathbf{x}_t)<0$, are deemed to be unsafe. We wish to keep the safety violation rate, i.e., the fraction of unsafe solutions attempted during the optimization process, below some tolerated threshold.}
  \label{fig: intro flow}

\end{figure}

As an example, consider the problem of discovering pharmaceuticals for a particular condition (see, e.g., \cite{bengio2023gflownet}). A pharmaceutical company may try different molecules by carrying out costly trials with patients. Such trials would return not only an indication of the effectiveness of the candidate cure, but also an indication of possible side effects. A reasonable goal is that of finding a maximally effective compound, while minimizing the number of molecules that are found to have potential side effects during the optimization process.


Typical tools for the solution of black-box zero-th order optimization construct surrogates of the objective function that are updated as information is collected by the optimizer.  This can be done using tools from reinforcement learning, such as bandit optimization \cite{slivkins2019introduction}, or Bayesian optimization (BO) \cite{mockus1989global, frazier2018tutorial,Maggi2021bogp,eriksson2019scalable}.

Focusing on methods based on BO, prior art has introduced an optimization scheme -- referred to as \textsf{S{\footnotesize AFE}O{\footnotesize PT}} \cite{sui2015safeopt, Felix2016safeopt} --  that is guaranteed not to select \emph{any} unsafe solution with a controllable probability with respect to feedback noise. This theoretical guarantee is, however, only valid if the optimizer has access to information about the constraint function. In particular, reference \cite{sui2015safeopt, Felix2016safeopt} assumes that the constraint function belongs to a reproducible kernel Hilbert space (RKHS), and that it has a known finite RKHS norm. In practice, specifying such information may be difficult, since the constraint function is a priori unknown.

In this paper, a novel BO-based approach is introduced that satisfies safety requirements irrespective of properties of the constraint function. This guarantee is obtained at the cost of allowing for an arbitrary, controllable but non-zero, rate of violation of the safety constraint. The proposed method, referred to as \textsf{S{\footnotesize AFE}-B{\footnotesize OCP}}, builds on online conformal prediction (CP) \cite{gibbs2021adaptive, feldman2023achieving}, and is specialized to the cases in which feedback on the safety constraint is either noiseless or noisy.

\subsection{Related Work}\label{ssec: related work}
Existing constrained sequential black-box zero-th order optimizers that leverage BO, collectively referred as Safe-BO schemes, target a strict safety requirement whereby no safety violations are allowed. Accordingly, all candidate solutions attempted by the optimizer must be safe \cite{sui2015safeopt,Felix2016safeopt, sui2018stagewise, matteo2019goose, rothfuss2023meta}. As mentioned in the previous subsection, such stringent safety requirements can only be guaranteed by making strong assumptions on the knowledge available regarding the safety constraint function.  In particular, all the existing works on Safe-BO, with a notable exception of \cite{rothfuss2023meta}, either assume knowledge of the smoothness properties of the constraint function when dealing with deterministic constraint function \cite{sui2015safeopt, sui2018stagewise, Felix2016safeopt, matteo2019goose}, or treating the constraint function as a random realization of a Gaussian process with a known kernel when dealing with random constraint function \cite{Felix2016safeopt}. 



When the mentioned assumptions or the surrogate model on the constraint function are invalid or infeasible, the existing methods cannot provide any formal safety guarantees. In order to mitigate this problem, reference \cite{rothfuss2023meta} proposed to apply meta-learning \cite{hospedales2021meta} to estimate a suitable surrogate model for the constraint function using  additional data that are assumed to be available from other, similar, optimization problems. However, no formal safety guarantees are available for the approach.

In a related line of work, the constrained BO approaches in \cite{gelbart2014cbo,lobato2016ces,Gardner2014cbo} target a constrained optimization problem, but allow unlimited safety violations during the optimization process. More recent references \cite{marco2020excursion,yu2022learning} considered an explicit budget of safety violations for probabilistic constraints, but did not provide any formal safety guarantees.


CP is a general framework for the calibration of statistical models \cite{vovk2005algorithmic}. CP methods can be applied to pre-trained machine learning models with the goal of ensuring that the model's outputs provide reliable estimates of their uncertainty.   There are two main classes of CP techniques: \emph{offline CP}, which leverages offline calibration data for this purpose \cite{vovk2005algorithmic, angelopoulos2021gentle}; and \emph{online CP}, which uses feedback on the reliability of past decisions to adjust the post-processing of model's outputs \cite{gibbs2021adaptive, feldman2023achieving}. In both cases, CP offers theoretical guarantees on the quality of the uncertainty quantification provided by the decisions of the system.

The relevance of \emph{online CP} for the problem of interest, illustrated in Fig. 1, is that, as the optimizer attempts multiple solutions over time, it needs to maintain an estimate of the constraint function. In order to ensure the safety of the candidate solutions selected by the optimizers, it is important that such estimates come with well-calibrated uncertainty intervals. In this paper, we leverage the theoretical guarantees of online CP in order to define novel BO-based safe optimization strategies.


The only existing combination of CP and BO we are aware of  are provided by \cite{Samuel2022bocp}, which apply \emph{offline} CP to BO for the solution of an \emph{unconstrained} optimization problem. The approach aims at improving the acquisition function while accounting for  observation noise that goes beyond the standard homoscedastic Gaussian assumption. These prior works do not address safety requirements.

\begin{table}[!t]
\caption{State of the art on Safe-BO against the proposed Safe-BOCP\label{table: comparison}}
\centering
\begin{tabular}{|c||c|c|}
\hline
&\makecell[c]{$\text{Safe-BO}$ \\ \cite{sui2015safeopt, sui2018stagewise, Felix2016safeopt, matteo2019goose, rothfuss2023meta}} & \text{\textsf{S{\footnotesize AFE}-B{\footnotesize OCP}}  (ours)} \\
\hline
\text{Target safety violation rate} & 0 & (0,1]\\
\hline
\text{Assumption-free safety guarantee} & \ding{55} & \ding{51} \\
\hline
\end{tabular}

\end{table}


\subsection{Main Contributions}\label{ssec: main contributions}

 In this paper, we introduce \textsf{S{\footnotesize AFE}-B{\footnotesize OCP}}, a novel BO-based optimization strategy for constrained black-box zero-th order problems with safety constraints. \textsf{S{\footnotesize AFE}-B{\footnotesize OCP}}  provides \emph{assumptions-free} guarantees on the safety level of the attempted candidate solutions, while enabling any non-zero target safety violation level. As summarized in Table  \ref{table: comparison}, this contrasts with the state-of-the-art papers \cite{sui2015safeopt, sui2018stagewise, Felix2016safeopt,matteo2019goose, rothfuss2023meta} that only target the most stringent safety constraint with no safety violations throughout the optimization process, while relying on strong assumptions on the constraint function \cite{sui2015safeopt, sui2018stagewise, Felix2016safeopt,matteo2019goose}.
 


To summarize, the main contributions of the paper are as follows:\\
\noindent $\bullet$     We introduce the deterministic \textsf{S{\footnotesize AFE}-B{\footnotesize OCP}} (\textsf{D-S{\footnotesize AFE}-B{\footnotesize OCP}}) algorithm, which assumes noiseless feedback on the constraint function and targets a flexible safety constraint on the average number of candidate solutions that are found to be unsafe. The approach is based on a novel combination of online CP and Safe-BO methods.\\
    \noindent $\bullet$ For the case in which feedback on the constraint function is noisy, we introduce   the probabilistic  \textsf{S{\footnotesize AFE}-B{\footnotesize OCP}} (\textsf{P-S{\footnotesize AFE}-B{\footnotesize OCP}}) algorithm, which targets a flexible safety constraint on the \emph{probability} that the average number of candidate solutions that are found to be unsafe exceeds a controllable threshold. The method relies on a ``caution-increasing'' back-off mechanism that compensates for the uncertainty on the safety feedback received from the system.  \\
   \noindent $\bullet$We prove that both \textsf{D-S{\footnotesize AFE}-B{\footnotesize OCP}} and \textsf{P-S{\footnotesize AFE}-B{\footnotesize OCP}} meet their target safety requirements irrespective of the properties of the constraint function.\\
   \noindent $\bullet$ We validate the performance of all the proposed methods and theorems on a synthetic data set and on real-world applications.

The rest of the paper is organized as follows. Sec. \ref{sec: problem}  formulates the constrained black-box zero-th order problem with safety constraints. The general framework of Safe-BO, as well as the representative, state-of-the-art, algorithm \textsf{S{\footnotesize AFE}O{\footnotesize PT}}, are reviewed in Sec. \ref{sec: SBO} and Sec. \ref{sec: safeopt}, respectively. The proposed \textsf{S{\footnotesize AFE}-B{\footnotesize OCP}} methods are  introduced in the following sections,  with \textsf{D-S{\footnotesize AFE}-B{\footnotesize OCP}} presented in Sec. \ref{sec: d-safe-bocp} and \textsf{P-S{\footnotesize AFE}-B{\footnotesize OCP}} described in Sec. \ref{sec: bocp}. Experimental results on synthetic dataset are provided in Sec. \ref{sec: experiments}, and Sec. \ref{sec: real world} demonstrates results on real-world applications. Finally, Sec. \ref{sec: conclusion} concludes the paper.

\section{Problem Formulation}\label{sec: problem}
In this section, we describe the constrained black-box zero-th order optimization problems for safety-critical scenarios studied in this work. Then, we introduce the general solution framework of interest in the next section, which is referred to as Safe-BO \cite{sui2015safeopt,sui2018stagewise, Felix2016safeopt,matteo2019goose, rothfuss2023meta}.

\subsection{Optimization Problem and Safety Constraint}\label{ssec: problem statement}
We focus on constrained optimization problems of the form
\begin{align}
    \max\limits_{\mathbf{x}\in\mathcal{X}} f(\mathbf{x}) \quad
    \text{s.t.}  \quad q(\mathbf{x})\geq 0, \label{eq: opt target}
\end{align}
where objective function $f(\mathbf{x})$ and constraint function $q(\mathbf{x})$ are real valued; and  $\mathcal{X}$ is some specified subset of the $d$-dimensional vector space $\mathbbm{R}^d$. Let $f^{\mathrm{opt}}$ denote the maximum value of the problem \eqref{eq: opt target}, which we assume to be finite. We also assume that the set of optimal solutions, achieving the optimal value $f^{\mathrm{opt}}$, is not empty. We write any optimal solution as $\mathbf{x}^{\text{opt}}\in\mathcal{X}$ with $f^{\mathrm{opt}}=f(\mathbf{x}^{\text{opt}})$. Furthermore, we assume that there is a known, non-empty, set $\mathcal{S}_0\subset \mathcal{X} $ of safe solutions, i.e., \begin{equation}\label{eq:safezero} \mathcal{S}_0 \subseteq \{\mathbf{x} \in \mathcal{X}\textrm{: } q(\mathbf{x})\geq 0 \}. \end{equation} This subset may be as small as a single safe solution $\mathbf{x}_0$ with $q(\mathbf{x}_0) \geq 0$, i.e., $\mathcal{S}_0=\{\mathbf{x}_0\}$.

We address the optimization problem \eqref{eq: opt target} under the following conditions. \\
\noindent $\bullet$ \emph{Zero-th-order black-box access}: The real-valued objective function $f(\mathbf{x})$ and constraint function  $q(\mathbf{x})$ are a priori unknown, and only accessible as zero-th-order black boxes. This implies that, given a candidate solution  $\mathbf{x}$,  the optimizer can evaluate both functions, obtaining the respective values $f(\mathbf{x})$ and  $q(\mathbf{x})$. In practice, the evaluations are often noisy, resulting in the observation of noisy values $\tilde{f}(\mathbf{x})$ and $\tilde{q}(\mathbf{x})$. No other information, such as gradients, is obtained by the optimizer about the functions.\\
\noindent $\bullet$ \emph{Efficient optimization}: The optimizer wishes to minimize the number of accesses to both functions $f(\mathbf{x})$ and $q(\mathbf{x})$, while producing a feasible and close-to-optimal solution $\mathbf{x}^*\in\mathcal{X}$. That is, we wish for the optimizer to output a vector $\mathbf{x}^*\in\mathcal{X}$ that satisfies the constraint $q(\mathbf{x}^*)\geq 0$, with an objective value $f(\mathbf{x}^*)$ close to the maximum value $f^{\mathrm{opt}}$. The performance of the optimizer can be measured by the \emph{optimality ratio} 
\begin{equation} \Delta f(\mathbf{x}^*)=\frac{f(\mathbf{x}^*)}{f^{\mathrm{opt}}}.\label{eq: optimality ratio}
\end{equation}
\\
\noindent $\bullet$ \emph{Safety}: Interpreting the inequality $q(\mathbf{x})\geq0$ as a safety constraint, we consider choices of the optimization variable $\mathbf{x}\in\mathcal{X}$ that result in a negative value of the constraint function $q(\mathbf{x})$ to be \emph{unsafe}, unless the number of such violations of the constraint are kept below a threshold. Accordingly, we will require that the number of evaluations of the constraint function $q(\mathbf{x})$ that result in a violation of the inequality $q(\mathbf{x})\geq 0$ to be no larger than a pre-determined value. We will formalize this constraint next by describing the general operation of the optimizer.

\subsection{Sequential Surrogate-Based Safe Optimization} \label{ssec: sequential safe bo}

Starting from a given solution $\mathbf{x}_0\in\mathcal{S}_0$ \eqref{eq:safezero}, the optimizer sequentially produces \emph{candidate solutions} $\mathbf{x}_1,...,\mathbf{x}_T \in \mathcal{X}$ across $T$ \emph{trials} or \emph{iterations}. At each iteration $t$, the optimizer receives noisy observations of the objective value $f(\mathbf{x}_t)$ as 
\begin{align}
    y_t=f(\mathbf{x}_t)+\epsilon_{f,t}, \label{eq: noisy f obs}
\end{align}
as well as a noisy observation of the constraint value $q(\mathbf{x}_t)$ as 
\begin{align}
    z_t=q(\mathbf{x}_t)+\epsilon_{q,t}, \label{eq: noisy q obs}
\end{align}
where the observation noise for the objective, $\epsilon_{f,t}\sim \mathcal{N}(0,\sigma_f^2)$, is Gaussian with variance $\sigma_f^2$, while the observation noise for the constraint, $\epsilon_{q,t}$, can follow any distribution provided that it has a known upper bound on the one-sided right-tail probability (see Assumption \ref{ap: assump 1} in Sec. \ref{ssec: p-safe-bocp adaptive scaling} for details).

We focus on optimizers that maintain \emph{surrogate models} of functions $f(\mathbf{x})$ and $q(\mathbf{x})$ in order to select the next iterate. To elaborate, let us write as $\mathcal{O}_t$ the overall history of past iterates $(\mathbf{x}_0,...,\mathbf{x}_t)$ and past observations $(y_0,z_0,...,y_t,z_t)$ at the end of the $t$-th iteration, i.e., \begin{equation}\mathcal{O}_t=(\mathbf{x}_0,...,\mathbf{x}_t,y_0,...,y_t,z_0,...,z_t).\end{equation} As we detail in the next section, the optimizer maintains probability distributions $p(f|\mathcal{O}_t)$ and $p(q|\mathcal{O}_t)$ on the functions $f(\mathbf{x})$ and $q(\mathbf{x})$ across all values $\mathbf{x}\in\mathcal{X}$ based on the available information $\mathcal{O}_t$. The distributions $p(f|\mathcal{O}_t)$ and $p(q|\mathcal{O}_t)$ summarize the belief of the optimizer regarding the values of the two functions.

At the next iteration $t+1$, the optimizer leverages the distributions $p(f|\mathcal{O}_t)$ and $p(q|\mathcal{O}_t)$ to obtain iterate $\mathbf{x}_{t+1}$ as follows.\\
\noindent $\bullet$ \emph{Safe set}: Using distribution $p(q|\mathcal{O}_t)$, the optimizer identifies a safe set $\mathcal{S}_{t+1}\subseteq\mathcal{X}$, containing solutions $\mathbf{x}\in\mathcal{X}$ deemed by the optimizer to be safe, i.e., to satisfy the constraint $q(\mathbf{x})\geq0$.\\
\noindent $\bullet$ \emph{Acquisition}: Using distributions $p(f|\mathcal{O}_t)$ and $p(q|\mathcal{O}_t)$, the optimizer selects the next iterate $\mathbf{x}_{t+1}\in\mathcal{S}_{t+1}$, with the aim of maximizing the likelihood of obtaining a large, i.e., close to 1, optimality ratio \eqref{eq: optimality ratio}.


\subsection{Safety Constraints}\label{ssec: safety constraints}
We now formalize the safety constraint by distinguishing the cases in which the observations \eqref{eq: noisy q obs} of constraint function $q(\mathbf{x})$ are: (\emph{i}) \emph{noiseless}, i.e., we have $z_t=q(\mathbf{x}_t)$ in \eqref{eq: noisy q obs} with noise power $\sigma_q^2=0$; and (\emph{ii)} \emph{noisy}, i.e., we have a positive observation noise power $\sigma_q^2 > 0$ in \eqref{eq: noisy q obs}.

\subsubsection{Deterministic Safety Constraint} 
Noiseless observations of the constraint function values allow the optimizer to keep track of the number of iterates $\mathbf{x}_t$ that result in violations of the non-negativity constraint in problem (\ref{eq: opt target}). Accordingly, with $\sigma^2_q=0$, we impose that the non-negativity constraint $q(\mathbf{x}_t)\geq0$ be violated no more than a tolerated fraction $\alpha\in[0,1]$ of the $T$ iterations. Specifically, given a \emph{target violation rate} $\alpha\in[0,1]$, this results in the deterministic safety requirement
\begin{align}
    \label{eq:goal}
    \text{violation-rate}(T) := \frac{1}{T}\sum_{t=1}^T \mathbbm{1}(q(\mathbf{x}_t) < 0) \leq \alpha,
\end{align}
where $\mathbbm{1}(\cdot)$ is the indicator function, i.e., we have $\mathbbm{1}(\text{true})=1$ and 
 $\mathbbm{1}(\text{false})=0$. Therefore, in this first case, we target the maximization of function $f(x)$ subject to the safety constraint (\ref{eq:goal}) on the optimization process.
 

\subsubsection{Probabilistic Safety Constraint} 
In the presence of observation noise on the constraint, i.e., with a positive observation noise power $\sigma_q^2>0$, the optimizer cannot guarantee the deterministic constraint \eqref{eq:goal}.  Rather, targeting problem (\ref{eq: opt target}), the optimizer can only aim at ensuring that the constraint \eqref{eq:goal} be satisfied with a probability no smaller than a \emph{target reliability level} $1-\delta$, with $\delta\in(0,1]$. This results in the \emph{probabilistic} safety constraint
\begin{equation} 
\Pr(\text{violation-rate}(T)\leq \alpha)\geq1-\delta,\label{eq: probabilistic goal}
\end{equation} in which the probability is taken with respect to the observation noise variables $\{\epsilon_{q,t}\}_{t=1}^T$ for the constraint function $q(\mathbf{x})$ in \eqref{eq: noisy q obs}. Therefore, in this second case, we target the maximization of function $f(x)$ subject to the safety constraint (\ref{eq: probabilistic goal}) on the optimization process.


\section{Safe Bayesian Optimization}\label{sec: SBO}
We adopt BO as the underlying surrogate-based optimization strategy. When deployed to address the problem of safe black-box optimization defined in the previous section, BO-based schemes are referred to collectively as \emph{Safe-BO} \cite{sui2015safeopt, sui2018stagewise,Felix2016safeopt, matteo2019goose, rothfuss2023meta}. As illustrated in Fig. \ref{fig: safebo framework}, Safe-BO models objective function $f(\mathbf{x})$ and constraint function $q(\mathbf{x})$ by using independent Gaussian processes (GPs) as surrogate models, producing the distributions $p(f|\mathcal{O}_t)$ and $p(q|\mathcal{O}_t)$ introduced in Sec. \ref{ssec: sequential safe bo}. In this section, we first review background material on GPs in Sec. \ref{ssec: gp}. Then, we discuss a general approach to define safe sets $\mathcal{S}_{t+1}$ on the basis of the current distribution $p(q|\mathcal{O}_t)$ in Sec. \ref{ssec: safe set}. 
\begin{figure}[t]

  \centering
  \centerline{\includegraphics[scale=0.54]{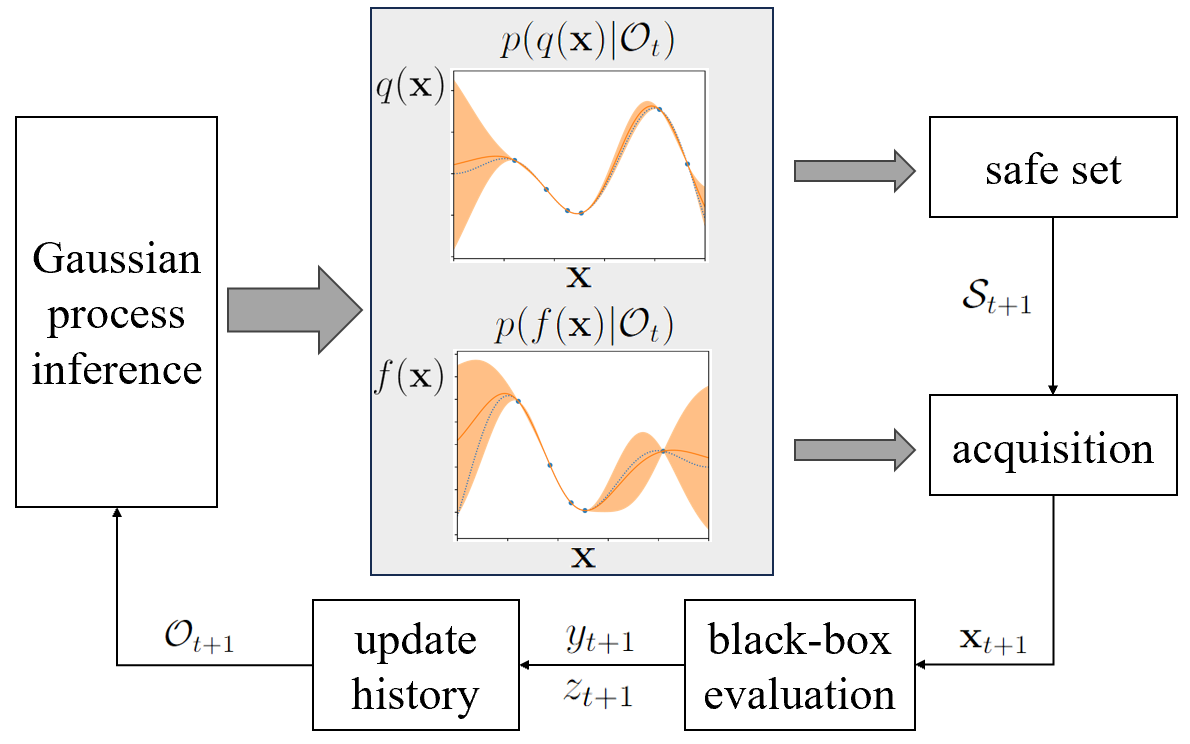}}
  \caption{Block diagram of Safe-BO schemes consisting of the main steps of safe set creation, producing the safe set $\mathcal{S}_{t+1}$, and of acquisition, selecting the next iterate $\mathbf{x}_{t+1}$.}
  \label{fig: safebo framework}

\end{figure}

\subsection{Gaussian Process}\label{ssec: gp}
Consider an unknown scalar-valued function $g(\mathbf{x})$ with input $\mathbf{x}\in\mathbbm{R}^d$. GP models such a function by assuming that, for any collection $(\mathbf{x}_1,...,\mathbf{x}_N)$ of inputs, the corresponding outputs $(g(\mathbf{x}_1),...,g(\mathbf{x}_1))$ follow a multivariate Gaussian distribution. The Gaussian distribution is characterized by a mean function $\mu(\mathbf{x})$ with $\mathbf{x}\in\mathbbm{R}^d$, and kernel function $\kappa(\mathbf{x},\mathbf{x}')$ for $\mathbf{x},\mathbf{x}'\in\mathbbm{R}^d$ \cite{Rasmussen2004}. An examples of a kernel function is the radial basis function (RBF) kernel\begin{align}\kappa(\mathbf{x},\mathbf{x}')=\exp(-h||\mathbf{x}-\mathbf{x}'||^2), \label{eq: rbf kernel}
\end{align}
which depends on a  bandwidth parameter $h>0$. Specifically, for given inputs $(\mathbf{x}_1,...,\mathbf{x}_N)$, collectively denoted as $\mathbf{X}$, the output vector $(g(\mathbf{x}_1),...,g(\mathbf{x}_1))$ follows a Gaussian distribution $\mathcal{N}(\boldsymbol{\mu}(\mathbf{X}),\mathbf{K}(\mathbf{X}))$, with $N\times1$ mean  vector $\boldsymbol{\mu}(\mathbf{X})=[\mu(\mathbf{x}_1),...,\mu(\mathbf{x}_N)]^{\sf T}$, and  $N\times N$ covariance matrix $\mathbf{K}(\mathbf{X})$ with each $(n,n')$-th entry given by $\kappa(\mathbf{x}_n,\mathbf{x}_{n'})$.

Assume that the output $g(\mathbf{x})$ is observed in the presence of independent Gaussian noise as
\begin{align}
    y=g(\mathbf{x})+\epsilon, \label{eq: noisy g obs}
\end{align}
with $\epsilon\sim \mathcal{N}(0,\sigma^2)$. We write as $\mathbf{y}=[y_1,..., y_N]^{\sf T}$ 
 the $N\times 1$ vector collecting the noisy outputs \eqref{eq: noisy g obs} for inputs $(\mathbf{x}_1,...,\mathbf{x}_N)$. An important property of GPs is that, given the history $\mathcal{O}=(\mathbf{X},\mathbf{y})$ of previous observations $\mathbf{y}$ for inputs $\mathbf{X}$, the posterior distribution $p(y|\mathbf{x},\mathcal{O})$ of a new output $y$ corresponding to any input $\mathbf{x}$ has a Gaussian distribution with mean $\mu(\mathbf{x}|\mathcal{O})$ and variance $\sigma^2(\mathbf{x}|\mathcal{O})$, i.e., 
\begin{subequations}
    \begin{equation}
    p(g(\mathbf{x})|\mathcal{O})= \mathcal{N}(\mu(\mathbf{x}|\mathcal{O}), \sigma^2(\mathbf{x}|\mathcal{O})),\label{eq: GP posterior}
    \end{equation}
    \begin{equation}
    \text{with}\quad \mu(\mathbf{x}|\mathcal{O})=\mu(\mathbf{x})+\boldsymbol{\kappa}(\mathbf{x})^{\sf T}(\mathbf{K}(\mathbf{X})+\sigma^2\mathbf{I}_N)^{-1}(\mathbf{y}-\boldsymbol{\mu}(\mathbf{X})),\label{eq: GP mean}\end{equation}
    \begin{equation} \text{and}\quad\sigma^2(\mathbf{x}|\mathcal{O})=\kappa(\mathbf{x},\mathbf{x})-\boldsymbol{\kappa}(\mathbf{x})^{\sf T}(\mathbf{K}(\mathbf{X})+\sigma^2\mathbf{I}_N)^{-1}\boldsymbol{\kappa}(\mathbf{x}),\quad\label{eq: kernel}\end{equation}
\end{subequations}
with $N\times1$ covariance vector $\boldsymbol{\kappa}(\mathbf{x})=[\kappa(\mathbf{x},\mathbf{x}_1),\hdots,\kappa(\mathbf{x},\mathbf{x}_N)]^{\sf T}$ and identity matrix $\mathbf{I}_N\in\mathbbm{R}^{N\times N}$.

\subsection{Credible Intervals and Safe Set}
\label{ssec: safe set}
Let us return to the operation of sequential optimizers based on BO. As explained in the previous section, at the end of iteration $t$, the optimizer has attempted solutions $(\mathbf{x}_1,...,\mathbf{x}_t)$, which are collectively referred to as $\mathbf{X}_t$. For these inputs, it has observed the noisy values $\mathbf{y}_t=[y_1,...,y_t]^{\sf T}$ in  \eqref{eq: noisy f obs} of the objective function, as well as the noisy values $\mathbf{z}_t=[z_1,...,z_t]^{\sf T}$ in  \eqref{eq: noisy q obs} for the constraint function. As we reviewed in Sec. \ref{ssec: gp}, GPs allow the evaluation of the posterior distributions \begin{equation}\label{eq:posterior_f}p(f(\mathbf{x})|\mathcal{O}_t)=p(f(\mathbf{x})|\mathbf{X}_t,\mathbf{y}_t)\end{equation} and \begin{equation}\label{eq:posterior_q}p(q(\mathbf{x})|\mathcal{O}_t)=p(q(\mathbf{x})|\mathbf{X}_t,\mathbf{z}_t)\end{equation} for a new candidate solution $\mathbf{x}$,  given the history $\mathcal{O}_t=(\mathbf{X}_t,\mathbf{y}_t,\mathbf{z}_t)$ consisted of the previous attempts $\mathbf{X}_t$ and its corresponding noisy observations $\mathbf{y}_t$ and $\mathbf{z}_t$. As we discuss next, these posterior distributions are used by Safe-BO methods to construct \emph{credible intervals}, which quantify the residual uncertainty on the values of functions $f(\mathbf{x})$ and $q(\mathbf{x})$ at any candidate solution $\mathbf{x}$.

Introducing a \emph{scaling parameter} $\beta_{t+1}>0$, the credible interval for the value of the  objective function $f(\mathbf{x})$  for input $\mathbf{x}$ at the end of iteration $t$, or equivalently at the beginning of iteration $t+1$, is defined by lower bound $f_l(\mathbf{x}|\mathcal{O}_t)$ and upper bound $f_u(\mathbf{x}|\mathcal{O}_t)$ given by 
\begin{align}
    \mathcal{I}_f(\mathbf{x}|\mathcal{O}_t)=[&f_l(\mathbf{x}|\mathcal{O}_t),f_u(\mathbf{x}|\mathcal{O}_t)]\nonumber\\=[&\mu_f(\mathbf{x}|\mathbf{X}_t,\mathbf{y}_t)-\beta_{t+1} \sigma_f(\mathbf{x}|\mathbf{X}_t,\mathbf{y}_t),\nonumber\\&\mu_f(\mathbf{x}|\mathbf{X}_t,\mathbf{y}_t)+\beta_{t+1}\sigma_f(\mathbf{x}|\mathbf{X}_t,\mathbf{y}_t)],\label{eq: f credible interval}
\end{align} where the mean  $\mu_f(\mathbf{x}|\mathbf{X}_t,\mathbf{y}_t)$ and the standard deviation $\sigma_f(\mathbf{x}|\mathbf{X}_t,\mathbf{z}_t)$ are defined as in \eqref{eq: GP mean} and \eqref{eq: kernel}, respectively. 
In a similar manner, the credible interval for the constraint function $q(\mathbf{x})$ is defined as
\begin{align}
    \mathcal{I}_q(\mathbf{x}|\mathcal{O}_t)=[&q_l(\mathbf{x}|\mathcal{O}_t),q_u(\mathbf{x}|\mathcal{O}_t)]\nonumber\\=[&\mu_q(\mathbf{x}|\mathbf{X}_t,\mathbf{z}_t)-\beta_{t+1} \sigma_q(\mathbf{x}|\mathbf{X}_t,\mathbf{z}_t),\nonumber\\&\mu_q(\mathbf{x}|\mathbf{X}_t,\mathbf{z}_t)+\beta_{t+1}\sigma_q(\mathbf{x}|\mathbf{X}_t,\mathbf{z}_t)],\label{eq: q credible interval}
\end{align}where the mean  $\mu_q(\mathbf{x}|\mathbf{X}_t,\mathbf{y}_t)$ and the standard deviation $\sigma_q(\mathbf{x}|\mathbf{X}_t,\mathbf{z}_t)$ are also defined as in \eqref{eq: GP mean} and \eqref{eq: kernel}, respectively.

Under the Gaussian model assumed by GP, the intervals \eqref{eq: f credible interval} and \eqref{eq: q credible interval} include the true function values $f(\mathbf{x})$ and $q(\mathbf{x})$ for a given input $\mathbf{x}$ with probability
\begin{align}
    \mathrm{P}(\beta_{t+1})=2F(\beta_{t+1})-1, \label{eq: Q function}
\end{align}
where $F(\cdot)$ is the cumulative distribution function (CDF) of standard Gaussian random variable $F(z) = \Pr(Z \leq z)$ with  $Z \sim \mathcal{N}(0,1)$. 
Therefore, the lower bounds $f_l(\mathbf{x}|\mathcal{O}_t)$ and $q_l(\mathbf{x}|\mathcal{O}_t)$ in the credible intervals \eqref{eq: f credible interval} and \eqref{eq: q credible interval}, respectively, serve as \emph{pessimistic} estimates of the objective and constraint values at the confidence level defined by probability $\mathrm{P}(\beta_{t+1})$. Furthermore, under the same confidence level, the upper bounds $f_u(\mathbf{x}|\mathcal{O}_t)$ and $q_u(\mathbf{x}|\mathcal{O}_t)$ in \eqref{eq: f credible interval} and \eqref{eq: q credible interval} describe \emph{optimistic} estimates of the objective and constraint values, respectively. That said, it is important to stress that, since the Gaussian model assumed by GP is generally \emph{misspecified}, there is no guarantee on the actual probability that the credible intervals $\mathcal{I}_f(\mathbf{x}|\mathcal{O}_t)$ and $\mathcal{I}_q(\mathbf{x}|\mathcal{O}_t)$ include the true values $f(\mathbf{x})$ and $q(\mathbf{x})$. These intervals, in fact, are guaranteed to include the true functions values with probability $\mathrm{P}(\beta_{t+1})$ only under the GP model.

In order to meet the safety requirement \eqref{eq:goal} or \eqref{eq: probabilistic goal}, Safe-BO methods define a \emph{safe set} of candidate solutions $\mathbf{x}\in\mathcal{X}$ that are likely to  satisfy the constraint $q(\mathbf{x})\geq 0$. To this end, the optimizer selects the scaling factor $\beta_{t+1}$ so as to ensure some desired ``safety'' probability $\mathrm{P}(\beta_{t+1})$. Then, leveraging the GP model, Safe-BO methods adopt the pessimistic estimate of the value of constraint function given by  $q_l(\mathbf{x}|\mathcal{O}_t)$ in (\ref{eq: q credible interval}) as a conservative estimate of the constraint function. Accordingly, the safe set $\mathcal{S}_{t+1}$ is defined as the set of all feasible solutions $\mathbf{x}\in\mathcal{X}$ for which the conservative estimate $q_l(\mathbf{x}|\mathcal{O}_t)$ of constraint function $q(\mathbf{x})$ predicts the solution $\mathbf{x}$ to be safe, i.e.,  
\begin{align}
    \mathcal{S}_{t+1}=\mathcal{S}(\mathcal{O}_t|\beta_{t+1})= \{\mathbf{x}\in\mathcal{X}:q_l(\mathbf{x}|\mathcal{O}_t)\geq0\}  \cup  \mathcal{S}_0. \label{eq: safebo safe set}
\end{align}
The safe set  includes the known initial set $\mathcal{S}_0$ of safe solutions  in (\ref{eq:safezero}), ensuring a non-empty safe set \cite{Felix2016safeopt}.

Safe-BO schemes choose as the first solution $\mathbf{x}_0$ a point randomly selected from the initial safe set $\mathcal{S}_0$. For the following  iterations, while all Safe-BO schemes adopt the same definition of the safe set  \eqref{eq: safebo safe set}, the realization of the acquisition process selecting the next iterate $\mathbf{x}_{t+1}$ differentiates the schemes proposed in prior \cite{sui2015safeopt, sui2018stagewise, Felix2016safeopt,matteo2019goose, rothfuss2023meta}. In the next section, we specifically describe the operation of \textsf{S{\footnotesize AFE}O{\footnotesize PT}} \cite{sui2015safeopt,Felix2016safeopt}.

\section{\textsf{S{\footnotesize AFE}O{\footnotesize PT}}}\label{sec: safeopt}
In this section, we review \textsf{S{\footnotesize AFE}O{\footnotesize PT}} \cite{sui2015safeopt,Felix2016safeopt} a representative state-of-the-art Safe-BO method, which will serve as a reference for the proposed \textsf{S{\footnotesize AFE}-B{\footnotesize OCP}} strategies introduced in the next section. 

\subsection{Scope and Working Assumptions}\label{ssec: rkhs}
\textsf{S{\footnotesize AFE}O{\footnotesize PT}} addresses problem \eqref{eq: opt target} under a strict version of the probabilistic safety constraint \eqref{eq: probabilistic goal} with target violation rate $\alpha=0$ and arbitrary target reliability level $1-\delta$. In order to allow for a zero violation rate ($\alpha=0$) to be a feasible goal, \textsf{S{\footnotesize AFE}O{\footnotesize PT}} makes the assumption that the constraint function $q(\mathbf{x})$ in \eqref{eq: opt target} lies in the RKHS $\mathcal{H}_{\kappa}$ associated with the same kernel function $\kappa(\mathbf{x},\mathbf{x}')$ assumed by GP inference (see Sec. \ref{ssec: gp}). In this sense, the model adopted by GP is assumed by \textsf{S{\footnotesize AFE}O{\footnotesize PT}} to be well specified.

Formally, the mentioned assumption made by \textsf{S{\footnotesize AFE}O{\footnotesize PT}} enforces that the function can be expressed as
\begin{align}
q(\mathbf{x}) = \sum_{i=1}^m a_i \kappa(\mathbf{x}, \mathbf{x}_i) \label{eq: rkhs for q}
\end{align}
for some vectors $\{\mathbf{x}_i\in\mathbbm{R}^d\}_{i=1}^m$, real coefficients $\{a_i\}_{i=1}^m$, and integer $m$. For a function $q(\mathbf{x})$ of the form \eqref{eq: rkhs for q}, the \emph{squared RKHS norm} is defined as
\begin{align}
    ||q||^2_{\kappa}=\sum_{i=1}^{m}\sum_{j=1}^ma_ia_j\kappa(\mathbf{x}_i,\mathbf{x}_j).\label{eq:RKHS_norm}
\end{align}
Furthermore, a useful property of constraint function $q(\mathbf{x})$ in RKHS $\mathcal{H}_{\kappa}$ is that it is upper bounded by a function of their squared RKHS norm as
\begin{align}
    \label{eq:CS_ineq}
    |q(\mathbf{x})| \leq  \kappa(\mathbf{x},\mathbf{x})^{1/2} ||q||_{\kappa}
\end{align}
for all values $\mathbf{x}$ in their domain. The property \eqref{eq:CS_ineq} is leveraged by \textsf{S{\footnotesize AFE}O{\footnotesize PT}} by assuming that the RKHS norm of the constraint function $q(\mathbf{x})$ is upper bounded by a known constant $B$, i.e.,
\begin{align}
    ||q||_{\kappa}\leq B. \label{eq: norm bound B}
\end{align}

\subsection{Safe Set Creation}\label{ssec: update credible interval}
Safe-BO determines the safe set $\mathcal{S}_{t+1}$ in \eqref{eq: safebo safe set} using the scaling parameter
\begin{align}
    \beta_{t+1}=B+4\sigma_q\sqrt{\gamma_{t}+1-\ln(\delta)},\label{eq: rkhs beta}
\end{align}
where $B$ is the constant appearing in the assumed upper bound \eqref{eq: norm bound B}; $\sigma_q^2$ is the known observation noise power in \eqref{eq: noisy q obs}; $1-\delta$ is the target reliability level in \eqref{eq: probabilistic goal}; and $\gamma_t$ is the \emph{maximal mutual information} between the true values $(q(\mathbf{x}_1),...,q(\mathbf{x}_t))$ of the constraint function and the corresponding $t$ noisy observations $(\mathbf{z}_1,...,\mathbf{z}_t)$ when evaluated under the model assumed by GP. This quantity can be evaluated as \cite{Felix2016safeopt} \begin{align}
    \gamma_t=\max_{\mathbf{X}'_t=(\mathbf{x}'_1,...,\mathbf{x}'_t)}\bigg(\frac{1}{2}\log\Big|\mathbf{I}_t+\sigma_q^{-2}\mathbf{K}_q(\mathbf{X}'_t) \Big|\bigg),\label{eq: max mi}
\end{align}
where $\mathbf{I}_t$ is the $t\times t$ identity matrix and $\mathbf{K}_q(\mathbf{X}'_t)$ is the $t\times t$ covariance matrix defined in Sec. \ref{ssec: gp}. Evaluating \eqref{eq: max mi} requires a maximization over all possible inputs sequences $\mathbf{X}'_t=(\mathbf{x}'_1,...,\mathbf{x}'_t)$, hence in practice it is often addressed via greedy algorithms (see, e.g., \cite{srinivas2012tit}). We also observe that, in the limit of no observation noise, i.e., as $\sigma_q\rightarrow0$, the scaling parameter \eqref{eq: rkhs beta} tends to $\beta_t = B$. 

By choosing the scaling parameter $\beta_{t+1}$ as in \eqref{eq: rkhs beta}, under the key assumption (\ref{eq: norm bound B}), all the decisions in the safe set $\mathcal{S}_{t+1}$ \eqref{eq: safebo safe set} can be proved to be safe with high probability \cite[Lemma~1]{Felix2016safeopt} (see also \cite[Theorem 6]{srinivas2012tit}).

\subsection{Acquisition Process}\label{ssec: ap}
In this section, we detail the acquisition process adopted by \textsf{S{\footnotesize AFE}O{\footnotesize PT}} to select the next iterate $\mathbf{x}_{t+1}$ within the safe set $\mathcal{S}_{t+1}$. 

To start, \textsf{S{\footnotesize AFE}O{\footnotesize PT}} defines the set of \emph{potential optimizers} $\mathcal{M}_{t+1}$ 
as the set of all possible solutions $\mathbf{x}\in\mathcal{S}_{t+1}$ that may increase the objective function. It also maintains a set of \emph{possible expanders} $\mathcal{G}_{t+1}$ as the set of safe solutions that can potentially increase the size of the safe set $\mathcal{S}_{t+1}$ if selected. Then, given the potential optimizers $\mathcal{M}_{t+1}$ and the possible expanders $\mathcal{G}_{t+1}$, \textsf{S{\footnotesize AFE}O{\footnotesize PT}} chooses the solution $\mathbf{x} \in \mathcal{M}_{t+1} \cup \mathcal{G}_{t+1}$ that maximally reduces the larger uncertainty implied by the credible intervals \eqref{eq: f credible interval} and \eqref{eq: q credible interval}, i.e., 
\begin{align}
    \mathbf{x}_{t+1} = \argmax_{\mathbf{x} \in \mathcal{M}_{t+1} \cup \mathcal{G}_{t+1}} \max \{ \sigma_f(\mathbf{x}|\mathcal{O}_t), \sigma_q(\mathbf{x}|\mathcal{O}_t) \}. \label{eq: safeopt next solution}
\end{align}

We now describe the construction of sets $\mathcal{M}_{t+1}$ and $\mathcal{G}_{t+1}$. For the first, let us recall that the lower bound $f_l(\mathbf{x}|\mathcal{O}_t)$ in the credible interval \eqref{eq: f credible interval} can be viewed as a pessimistic estimate of the objective $f(\mathbf{x})$, while the upper bound $f_u(\mathbf{x}|\mathcal{O}_t)$ can be interpreted as an optimistic estimate of the same value. The set of potential optimizers, $\mathcal{M}_{t+1}$, includes all safe solutions $\mathbf{x}\in \mathcal{S}_{t+1}$ for which the optimistic estimate $f_u(\mathbf{x}|\mathcal{O}_t)$ is larger than the best pessimistic estimate $f_l(\mathbf{x}|\mathcal{O}_t)$ for all safe solutions $\mathbf{x}\in\mathcal{S}_{t+1}$. This set can be expressed mathematically as 
\begin{align}
    \mathcal{M}_{t+1}=\bigg\{\mathbf{x}\in \mathcal{S}_{t+1}\Big|f_u(\mathbf{x}|\mathcal{O}_t)\geq\max\limits_{\mathbf{x}'\in S_{t+1}}f_l(\mathbf{x}'|\mathcal{O}_t)\bigg\}.\label{eq: ap potential set}
\end{align} Note that this set is non-empty, since it includes at least the solution $\mathbf{x}$ that maximizes the lower bound $f_l(\mathbf{x}|\mathcal{O}_t)$.

\begin{algorithm}[t!]
\caption{\textsf{S{\footnotesize AFE}O{\footnotesize PT}}}\label{table: safeopt}
\SetKwInOut{Input}{Input}
\Input{GP priors $(\mu_f(\mathbf{x}),\kappa_f(\mathbf{x},\mathbf{x}'))$ and $(\mu_q(\mathbf{x}),\kappa_q(\mathbf{x},\mathbf{x}'))$, initial safe set $\mathcal{S}_0$, initial observation $\mathcal{O}_0$, assumed RKHS norm bound $B$, total number of optimization iterations $T$} 
\SetKwInOut{Output}{Output}
\Output{Decision $\mathbf{x}^*$}\
\textbf{Initialize} scaling parameters $\{\beta_{t}\}_{t=1}^{T+1}$ using \eqref{eq: rkhs beta}, $\mathbf{x}_1 = \textsf{S{\footnotesize AFE}O{\footnotesize PT}}(\mathcal{O}_0|\beta_1)$ \\
\For{\emph{$t=1,...,T$}}{
Observe $y_{t}$ and $z_{t}$ from candidate solution $\mathbf{x}_t$\\
Update the observation history $\mathcal{O}_{t} = \mathcal{O}_{t-1} \cup \{\mathbf{x}_{t}, y_{t}, z_{t}\}$\\
Update GPs with $\mathcal{O}_{t}$ as in \eqref{eq:posterior_f} and \eqref{eq:posterior_q}\\
$\mathbf{x}_{t+1} =  \textsf{S{\footnotesize AFE}O{\footnotesize PT}}(\mathcal{O}_t|\beta_{t+1})$
}
\textbf{Return} final decision $\mathbf{x}^* = \arg \max_{\mathbf{x} \in \mathcal{S}_{T+1}} f_l(\mathbf{x}|\mathcal{O}_T)$  \\
------------------------------------------------------------------\\
\textsf{S{\footnotesize AFE}O{\footnotesize PT}}$(\mathcal{O}_t|\beta_{t+1})$:\\
\quad Create credible intervals $\mathcal{I}_f(\mathbf{x}|\mathcal{O}_t)$ and $\mathcal{I}_q(\mathbf{x}|\mathcal{O}_t)$ using $\beta_{t+1}$ as in \eqref{eq: f credible interval} and \eqref{eq: q credible interval}\\
\quad Obtain safe set $\mathcal{S}_{t+1}$ as in \eqref{eq: safebo safe set}\\
\quad Update the set of potential optimizers $\mathcal{M}_{t+1}$ as in \eqref{eq: ap potential set}\\
\quad Update the set of possible expanders $\mathcal{G}_{t+1}$ as in \eqref{eq: ap expanders}\\
\quad \textbf{Return} the next iterate $\mathbf{x}_{t+1}$ in accordance to \eqref{eq: safeopt next solution}\\
\end{algorithm}

The set $\mathcal{M}_{t+1}$ accounts only for the objective value to select solutions from the safe set $\mathcal{S}_{t+1}$. In contrast, the set of possible expanders considers the potential impact of a selected candidate solution on the safe set. To formalize this concept, let us write $\mathcal{S}_{t+2}(\mathbf{x})$ for the safe set \eqref{eq: safebo safe set} evaluated by extending the current history $\mathcal{O}_t$ with the pair $(\mathbf{x},q_u(\mathbf{x}|\mathcal{O}_t))$ of candidate solution $\mathbf{x}$ and corresponding hypothetical observation of the optimistic value $q_u(\mathbf{x}|\mathcal{O}_t)$ of the constraint $q(\mathbf{x})$. Accordingly, we have
\begin{align}
    \mathcal{S}_{t+2}(\mathbf{x})=\mathcal{S}\Big(\mathcal{O}_t\cup(\mathbf{x},q_u(\mathbf{x}|\mathcal{O}_t))\Big|\beta_{t+1}\Big), \label{eq: S_t+2}
\end{align}
and the set of possible expanders is defined as
\begin{align}
    \mathcal{G}_{t+1}=\{\mathbf{x}\in \mathcal{S}_{t+1}: | \mathcal{S}_{t+2}(\mathbf{x}) \setminus \mathcal{S}_{t+1} |>0\}, \label{eq: ap expanders}
\end{align}
that is, as the set of all safe solutions that can potentially increase the size of the safe set. 

After $T$ trials, the final decision $\mathbf{x}^*$ is obtained by maximizing  the pessimistic estimate $f_l(\mathbf{x}|\mathcal{O}_T)$ of the objective function that is available after the last iteration over the safe set $\mathcal{S}_{T+1}$, i.e.,  \begin{equation} \mathbf{x}^* = \arg \max_{\mathbf{x} \in \mathcal{S}_{T+1}} f_l(\mathbf{x}|\mathcal{O}_T).\end{equation} The overall procedure of \textsf{S{\footnotesize AFE}O{\footnotesize PT}} is summarized in Algorithm \ref{table: safeopt}.

\subsection{Safety Property}\label{ssec: safeopt safety theorem}
\textsf{S{\footnotesize AFE}O{\footnotesize PT}} was shown in \cite{sui2015safeopt, Felix2016safeopt} to achieve the probabilistic safety constraint \eqref{eq: probabilistic goal} with $\alpha=0$, as long as the assumptions that the true constraint function $q(\mathbf{x})$ is of the form \eqref{eq: rkhs for q} and that the RKHS norm bound \eqref{eq: norm bound B} holds.
\begin{theorem}
    (Safety Guarantee of \textsf{S{\footnotesize AFE}O{\footnotesize PT}} \cite{Felix2016safeopt}) Assume that the RKHS norm of the true constraint function $q(\mathbf{x})$ is bounded by $B>0$ as in \eqref{eq: norm bound B}. By choosing the scaling parameter $\beta_{t+1}$ as in \eqref{eq: rkhs beta}, \textsf{S{\footnotesize AFE}O{\footnotesize PT}} satisfies the probabilistic safety constraint \eqref{eq: probabilistic goal} with $\alpha=0$. Furthermore, with ideal observations of the constraint function $q(\mathbf{x})$, i.e., $\sigma_q=0$, by choosing the scaling parameter as $\beta_{t+1}=B$, \textsf{S{\footnotesize AFE}O{\footnotesize PT}} meets the deterministic requirement \eqref{eq:goal} with $\alpha=0$.
    \label{theorem: safeopt theorem}
\end{theorem}
From Theorem \ref{theorem: safeopt theorem}, as long as the Gaussian model assumed by GP is well specified -- in the sense indicated by the RKHS form \eqref{eq: rkhs for q} with known norm upper bound $B$ in \eqref{eq: rkhs beta} -- \textsf{S{\footnotesize AFE}O{\footnotesize PT}} ensures safe optimization with a zero target violation rate $\alpha=0$. In practice, however, it is hard to set a value for the constant $B$. Therefore, for any fixed constant $B$, the resulting algorithm does not have formal guarantees in terms of safety \cite{rothfuss2023meta}.

\section{Deterministic Safe-BO via Online Conformal Prediction}
\label{sec: d-safe-bocp}




As we have reviewed in Sec. \ref{sec: safeopt}, in order to achieve a zero target violation rate $\alpha=0$ in the safety constraints \eqref{eq:goal} and \eqref{eq: probabilistic goal}, \textsf{S{\footnotesize AFE}O{\footnotesize PT}} assumes that the constraint function $q(\mathbf{x})$  belongs to a specific family of functions. Other Safe-BO methods \cite{sui2018stagewise, Felix2016safeopt, matteo2019goose} also require the same assumption to guarantee the safety constraint (see Sec. \ref{sec: intro}). 
In the following two sections, we will introduce \textsf{S{\footnotesize AFE}-B{\footnotesize OCP}}, a novel Safe-BO scheme that achieves the safety constraint requirements \eqref{eq:goal} or \eqref{eq: probabilistic goal} without requiring \emph{any} assumptions on the underlying constraint function $q(\mathbf{x})$. This goal is met at the cost of obtaining a non-zero, controllable, target violation rate $\alpha \in (0,1]$ in the deterministic safety requirement \eqref{eq:goal} and in the probabilistic safety requirement \eqref{eq: probabilistic goal}. This section focuses on the case in which observations \eqref{eq: noisy q obs}  of the constraint function are ideal, i.e., $\epsilon_{q,t}=0$, hence aiming at achieving the deterministic safety constraint \eqref{eq:goal}. The next section addresses the case with noisy observations on the constraint function.

\subsection{Adaptive Scaling via Noiseless Feedback on Safety}
As detailed in Sec. \ref{sec: SBO},  \textsf{S{\footnotesize AFE}O{\footnotesize PT}} fixes \emph{a priori} the scaling parameters $\beta_{1},...,\beta_{T}$ to be used when forming the safe set \eqref{eq: safebo safe set}, along with the set of potential optimizers \eqref{eq: ap potential set} and possible expanders \eqref{eq: ap expanders}, irrespective of the actual history $\mathcal{O}_t$ of past iterates $\mathbf{X}_t$ and observations $\mathbf{y}_t$ and $\mathbf{z}_t$. This is done by leveraging the mentioned assumptions on the constraint function \eqref{eq: rkhs for q}--\eqref{eq: norm bound B}. In contrast, not relying on any assumption on the constraint function $q(\mathbf{x})$, the proposed \textsf{S{\footnotesize AFE}-B{\footnotesize OCP}}  selects the scaling parameter $\beta_{t+1}$ adaptively based on the history $\mathcal{O}_t$ by leveraging ideas from online CP \cite{gibbs2021adaptive, feldman2023achieving}.


In order to support the adaptive selection of a scaling parameter $\beta_{t+1}$ that ensures the deterministic safety constraint \eqref{eq:goal}, \textsf{S{\footnotesize AFE}-B{\footnotesize OCP}} maintains an \emph{excess violation rate} variable $\Delta\alpha_{t+1}$ across the iterations $t=1,...,T$. The variable $\Delta\alpha_{t+1}$ compares the number of previous unsafe candidate solutions $\mathbf{x}_t'$ with $t'=1,...,t$  to a tolerable number that depends on the target violation rate $\alpha$. The main idea is to use the excess violation rate $\Delta\alpha_{t+1}$ to update the parameter $\beta_{t+1}$: A larger excess violation rate $\Delta\alpha_{t+1}$ calls for a larger value of $\beta_{t+1}$ so as to ensure a more pronounced level of pessimism in the evaluation of the safe set \eqref{eq: safebo safe set}. This forces the acquisition function  \eqref{eq: safeopt next solution} to be more conservative, driving down the excess violation rate towards a desired non-positive value.

\subsection{\textsf{D-S{\footnotesize AFE}-B{\footnotesize OCP}}}
To define the excess violation rate, we first introduce the \emph{safety error signal} \begin{align}
    \text{err}_t = \mathbbm{1}(z_t<0), \label{eq: d-err_t}
\end{align}
which yields $\text{err}_t=1$ if the last iterate $\mathbf{x}_t$ was found to be unsafe based on the observation $z_t=q(\mathbf{x}_t)$, and $\text{err}_t=0$ otherwise.  An important property of schemes, like \textsf{S{\footnotesize AFE}O{\footnotesize PT}} and \textsf{D-S{\footnotesize AFE}-B{\footnotesize OCP}}, that rely on the use of safe sets of the form \eqref{eq: safebo safe set} is that one can ensure a zero error signal $\text{err}_t=0$ by setting $\beta_t=\infty$. In fact, with this maximally cautious selection, the safe set $\mathcal{S}_t$ includes only the initial safe set $\mathcal{S}_0$ in \eqref{eq:safezero}, which consists exclusively of safe solutions.

\begin{figure}[t]

  \centering
  \centerline{\includegraphics[scale=0.52]{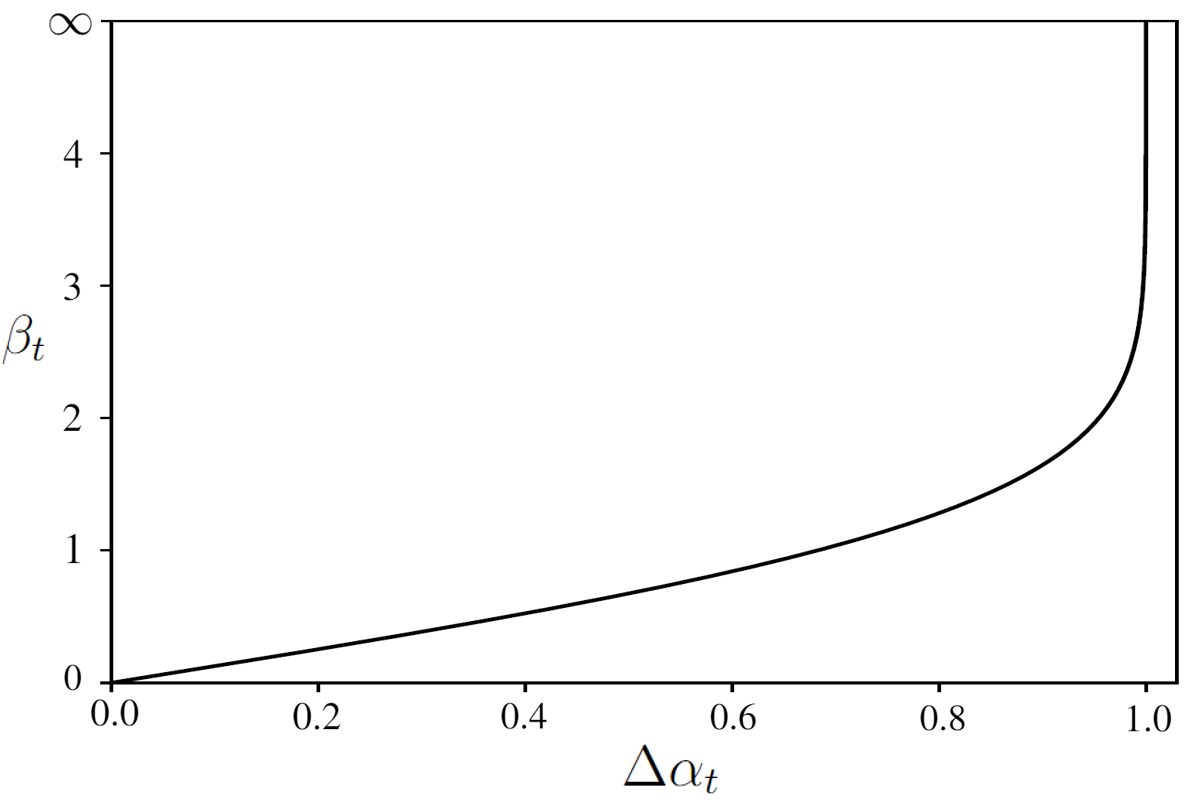}}
  \caption{Function $\beta_t=\varphi(\Delta\alpha_t)$ in \eqref{eq: q func}, which determines the scaling factor $\beta_t$ as a function of the excess violation rate $\Delta \alpha_t$.
  \label{fig: beta_alpha}}

\end{figure}

The excess violation rate  $\Delta\alpha_{t+1}$  measures the extent to which the average number of errors made so far, $1/t \cdot \sum_{t'=1}^t\text{err}_{t'}$,  exceeds an algorithmic target level $\alpha_{\text{algo}}$, which will be specified later.  Accordingly, the excess violation rate is updated as
\begin{align}
    \Delta\alpha_{t+1}=\Delta\alpha_{t}+\eta(\text{err}_{t}-\alpha_{\text{algo}}),\label{eq: update rule}
\end{align}
for a given update rate $\eta > 0$ and for any initialization $\Delta \alpha_1< 1$.  The relation between excess violation rate and the average number of errors becomes apparent by rewriting \eqref{eq: update rule} as\begin{align}
    \label{eq:unroll_dsafebocp}
    \Delta \alpha_{t+1} & = \Delta \alpha_1+ \eta \cdot \bigg( \sum_{t'=1}^t \text{err}_{t'} - \alpha_\text{algo}\cdot t\bigg) \nonumber \\ & =  \Delta \alpha_1+ \eta \cdot t \cdot \Big(\text{violation-rate}(t) - \alpha_\text{algo}\Big),
\end{align}
which is a linear function of  the difference between the violation rate up to time $t$ and the algorithmic target $\alpha_\text{algo}$. This implies  that the desired safety requirement (\ref{eq:goal}) can be equivalently imposed via the inequality 
\begin{align}
    \label{eq:unrolled_viol_rate_noiseless}
    \text{violation-rate}{(T)} & = \frac{\Delta \alpha_{T+1}-\Delta \alpha_1}{T\eta}  +\alpha_\text{algo} \leq \alpha.
\end{align} Therefore, controlling the violation rate requires us to make sure that the excess violation rate $\Delta\alpha_t$ does not grow too quickly with the iteration index $t$. 

\begin{algorithm}[t!]
\caption{\textsf{D-S{\footnotesize AFE}-B{\footnotesize OCP}}}\label{table: safe-bocp}
\SetKwInOut{Input}{Input}
\Input{GP priors $(\mu_f(\mathbf{x}),\kappa_f(\mathbf{x},\mathbf{x}'))$ and $(\mu_q(\mathbf{x}),\kappa_q(\mathbf{x},\mathbf{x}'))$, initial safe set $\mathcal{S}_0$, initial observation $\mathcal{O}_0$, total number of optimization iterations $T$, target violation rate $\alpha$, update rate $\eta>0$, initial excess violation rate $\Delta \alpha_1 < 1$}
\SetKwInOut{Output}{Output}
\Output{Decision $\mathbf{x}^*$}\
\textbf{Initialize} $\mathbf{x}_1 = \textsf{S{\footnotesize AFE}O{\footnotesize PT}}(\mathcal{O}_0|\beta_1)$ using $\beta_1 = \varphi(\Delta \alpha_1)$ \eqref{eq: q func}, algorithmic target level $\alpha_{\text{algo}}$ as in \eqref{eq:alpha_algo} \\
\For{\emph{$t=1,...T$}}{
Observe $y_{t}$ and $z_{t}$ from candidate solution $\mathbf{x}_t$ \\
Update the observation history 
$\mathcal{O}_{t}= \mathcal{O}_{t-1} \cup \{\mathbf{x}_{t},y_{t},z_{t}\}$\\
Update GPs with $\mathcal{O}_{t}$ as in \eqref{eq:posterior_f} and \eqref{eq:posterior_q} \\
Evaluate error signal $\text{err}_{t} = \mathbbm{1}(z_{t} < 0)$ as in \eqref{eq: d-err_t}\\
Update excess violation rate $\Delta\alpha_{t+1}= \Delta\alpha_{t}+\eta(\text{err}_{t}-\alpha_{\text{algo}})$ as in \eqref{eq: update rule}\\
Update scaling parameter $\beta_{t+1}=\varphi(\Delta \alpha_{t+1})$ using \eqref{eq: q func}\\
$\mathbf{x}_{t+1}=$\textsf{S{\footnotesize AFE}O{\footnotesize PT}}$(\mathcal{O}_t|\beta_{t+1})$ from Algorithm \ref{table: safeopt}
}
\textbf{Return} final decision $\mathbf{x}^* = \arg \max_{\mathbf{x} \in \mathcal{S}_{T+1}} f_l(\mathbf{x}|\mathcal{O}_T)$ \\
\end{algorithm}

Intuitively, as mentioned, in order to control the value of the excess violation rate $\Delta \alpha_t$, we need to select values of $\beta_t$ that increase with $\Delta \alpha_t$. To this end, as summarized in Algorithm 2, inspired by the approach introduced by \cite{feldman2023achieving} in the context of online CP, the proposed \textsf{D-S{\footnotesize AFE}-B{\footnotesize OCP}} sets the parameter $\beta_t$ as \begin{align}\label{eq:betavar}
    \beta_t = \varphi(\Delta \alpha_t),
\end{align}where we have defined function  \begin{align}
    \label{eq: q func}
    \varphi(\Delta \alpha_t) = F^{-1}((\text{clip}(\Delta \alpha_t)+1)/2),
\end{align} with $F^{-1}(\cdot)$ being the inverse of the function $F(\cdot)$ \eqref{eq: Q function}, i.e., the inverse CDF of standard Gaussian distribution, and $\text{clip}(\Delta \alpha_t)= \max\{\min\{ \Delta \alpha_t, 1\}, 0\}$  being the clipping function. An illustration of the function \eqref{eq: q func} can be found in Fig.~\ref{fig: beta_alpha}. Furthermore, we set the algorithmic target level as  
\begin{align}
\label{eq:alpha_algo}
\alpha_\text{algo}=\frac{1}{T-1}\bigg(T\alpha -1 -\frac{1}{\eta}+\frac{\Delta\alpha_1}{\eta}\bigg).    
\end{align} The overall procedure of \textsf{D-S{\footnotesize AFE}-B{\footnotesize OCP}} is summarized in Algorithm \ref{table: safe-bocp}.  We next prove that \textsf{D-S{\footnotesize AFE}-B{\footnotesize OCP}} meets the reliability requirement \eqref{eq:unrolled_viol_rate_noiseless}.


\subsection{Safety Guarantees}

\textsf{D-S{\footnotesize AFE}-B{\footnotesize OCP}} is guaranteed to meet the deterministic safety constraint (\ref{eq:goal}) (or equivalently (\ref{eq:unrolled_viol_rate_noiseless})), as summarized in the next theorem.

\begin{theorem}[Safety Guarantee of \textsf{D-S{\footnotesize AFE}-B{\footnotesize OCP}}]\label{theorem: safe-bocp}
    Under noiseless observations of the constraint function ($\sigma_q^2=0$), \textsf{D-S{\footnotesize AFE}-B{\footnotesize OCP}} satisfies the deterministic safety constraint \eqref{eq:goal} for any pre-determined target violation rate $\alpha \in (0,1]$.
\end{theorem}

\begin{proof}
Function (\ref{eq:betavar}) implements the following mechanism: When $\Delta \alpha_t\geq 1$, it returns $\beta_{t}=\infty$, i.e., \begin{equation}\label{eq:mechanism} \Delta \alpha_t\geq 1  \Rightarrow \beta_{t}=\infty.\end{equation} As discussed earlier in this section, this ensures a zero error signal $\text{err}_{t}=0$. With this mechanism in place, one can guarantee the upper  bound \begin{equation}
\label{eq:bound_excess_rate}
\Delta \alpha_{t+1} < 1+\eta (1-\alpha_{\textrm{algo}})
\end{equation} for all $t\geq1$ given the mentioned initialization $\Delta \alpha_1 < 1$. This is because a value $\Delta \alpha_t\geq 1$ would cause the update term in \eqref{eq: update rule} to $-\eta \alpha_{\mathrm{algo}}<0$, and hence the maximum value is attained when $\Delta \alpha_t$ is approaching, but smaller than, 1, and an unsafe decision is made, causing an update equal to $\eta(1-\alpha_\text{algo})$.

Plugging bound \eqref{eq:bound_excess_rate} back into \eqref{eq:unrolled_viol_rate_noiseless}, yields the upper bound on the violation rate
\begin{align}
    \text{violation-rate}{(T)}\leq \frac{1+\eta(1-\alpha_\text{algo})-\Delta \alpha_1}{T\eta}  +\alpha_\text{algo}.
\end{align}Therefore, by setting \eqref{eq:alpha_algo}, we finally verify that  the deterministic safety requirement \eqref{eq:unrolled_viol_rate_noiseless} is satisfied.
\end{proof}

\section{Probabilistic Safe-BOCP}\label{sec: bocp}

We now turn to the case in which the observations \eqref{eq: noisy q obs} of constraint function $q(\mathbf{x})$ are noisy. The main challenge in extending the approach proposed in the previous section is the fact that the error signal \eqref{eq: d-err_t} is an unreliable indication of whether candidate $\mathbf{x}_t$ is safe or not due to the presence of the observation noise $\epsilon_{q,t}$. Accordingly, we start by proposing an alternative way to measure the excess violation rate.

\subsection{\textsf{P-S{\footnotesize AFE}-B{\footnotesize OCP}} }\label{ssec: p-safe-bocp adaptive scaling}


To proceed, we assume that the observation noise $\epsilon_{q,t}$ in \eqref{eq: noisy q obs} has a known upper bound on the right-tail probability $\text{Pr}(\epsilon_{q,t}\geq \omega)$ for all $\omega\in\mathbbm{R}$. This basic assumption is also adopted in the robust CP literature \cite[Theorem 1]{feldman2023conformal}. In Sec.~\ref{ssec: est upper bound}, we will illustrate how to further alleviate this assumption by assuming access to noise samples.

\begin{assumption}\label{ap: assump 1}
    The constraint observation noise $\epsilon_{q,t}$, which is independent over $t=1,...,T$, has a  known upper bound $F^+(\omega)$ on its one-sided right-tail probability, i.e.,
\begin{align}
    \text{Pr}(\epsilon_{q,t}\geq \omega)\leq F^+(\omega)\label{eq: right tail pr}
\end{align}
for all $t=1,...,T$ and any $\omega \in \mathbb{R}$.
\end{assumption}


The main idea underlying the proposed \textsf{P-S{\footnotesize AFE}-B{\footnotesize OCP}} is to count as unsafe all solutions $\mathbf{x}_t$ for which the noisy observation $z_t=q(\mathbf{x}_t)+\epsilon_{q,t}$ in \eqref{eq: noisy q obs} is smaller than some back-off threshold $\omega_q > 0$. Specifically, we define the safety error signal as
\begin{align}
    \text{err}_t = \mathbbm{1}(z_t< \omega_q ), \label{eq: noisy err_t non Gaussian}
\end{align}
where the corresponding threshold $\omega_q$ is obtained as
\begin{align}
    \omega_q = \inf\{\omega \in \mathbb{R}:  F^+(\omega) \leq 1- (1-\delta)^{\frac{1}{T}} \}.\label{eq: omega q}
\end{align}
The threshold $\omega_q$ increases with the target reliability level $1-\delta$ in the probabilistic safety constraint \eqref{eq: probabilistic goal}. In fact, a larger target reliability level calls for more caution in determining whether a given observation $z_t$ of the constraint function is likely to indicate an unsafe solution or not. 

The rationale behind the definitions \eqref{eq: noisy err_t non Gaussian}-\eqref{eq: omega q} is formalized by the following lemma, which relates  the true violation rate \eqref{eq:goal} to the estimated violation rate $\sum_{t=1}^T\text{err}_t/T$ using the error signal \eqref{eq: noisy err_t non Gaussian}.

\begin{algorithm}[t!]
\caption{\textsf{P-S{\footnotesize AFE}-B{\footnotesize OCP}}}\label{table: p-safe-bocp}
\SetKwInOut{Input}{Input}
\Input{GP priors $(\mu_f(\mathbf{x}),\kappa_f(\mathbf{x},\mathbf{x}'))$ and $(\mu_q(\mathbf{x}),\kappa_q(\mathbf{x},\mathbf{x}'))$, initial safe set $\mathcal{S}_0$, initial observation $\mathcal{O}_0$, total number of optimization iterations $T$, target violation rate $\alpha$, update rate $\eta>0$, initial excess violation rate $\Delta \alpha_1 < 1$}
\SetKwInOut{Output}{Output}
\Output{Decision $\mathbf{x}^*$}\
\textbf{Initialize} $\mathbf{x}_1 = \textsf{S{\footnotesize AFE}O{\footnotesize PT}}(\mathcal{O}_0|\beta_1)$ using $\beta_1 = \varphi(\Delta \alpha_1)$ \eqref{eq: q func}, algorithmic target level $\alpha_{\text{algo}}$ as in \eqref{eq:alpha_algo} \\
\For{\emph{$t=1,...T$}}{
Observe $y_{t}$ and $z_{t}$ from candidate solution $\mathbf{x}_t$ \\
Update the observation history 
$\mathcal{O}_{t}= \mathcal{O}_{t-1} \cup \{\mathbf{x}_{t},y_{t},z_{t}\}$\\
Update GPs with $\mathcal{O}_{t}$ as in \eqref{eq:posterior_f} and \eqref{eq:posterior_q} \\
Evaluate \emph{cautious} error signal $\text{err}_{t} = \mathbbm{1}(z_{t} < \omega_q)$ as in \eqref{eq: noisy err_t non Gaussian} with $\omega_q$ obtained from \eqref{eq: omega q}\\
Update excess violation rate $\Delta\alpha_{t+1}= \Delta\alpha_{t}+\eta(\text{err}_{t}-\alpha_{\text{algo}})$ as in \eqref{eq: update rule}\\
Update scaling parameter $\beta_{t+1}=\varphi(\Delta \alpha_{t+1})$ using \eqref{eq: q func}\\
$\mathbf{x}_{t+1}=$\textsf{S{\footnotesize AFE}O{\footnotesize PT}}$(\mathcal{O}_t|\beta_{t+1})$ from Algorithm \ref{table: safeopt}
}
\textbf{Return} final decision $\mathbf{x}^* = \arg \max_{\mathbf{x} \in \mathcal{S}_{T+1}} f_l(\mathbf{x}|\mathcal{O}_T)$  \\
\end{algorithm}



\begin{lemma}[Estimated Violation Rate]\label{lemma: est vio rate}
    For any iterates $\mathbf{x}_1,...,\mathbf{x}_T$, the true violation rate in \eqref{eq:goal} is upper bounded by the accumulated error signal rate in \eqref{eq: noisy err_t non Gaussian} with probability $1-\delta$, i.e., 
    \begin{align}
        \Pr\bigg(\text{violation-rate($T$)}\leq\frac{1}{T}\sum_{t=1}^T\text{err}_t\bigg)&\geq (1-F^+(\omega))^T\nonumber\\&=1-\delta, \label{eq: type ii error}
    \end{align} \label{lemma: type ii error}in which the probability is taken with respect to the observation noise variables $\{\epsilon_{q,t}\}_{t=1}^T$ for the constraint function $q(\mathbf{x})$ in \eqref{eq: noisy q obs}.
\end{lemma}
\begin{proof}
When a candidate solution $\mathbf{x}_t$ is unsafe, i.e., when $q(\mathbf{x}_t)< 0$, the probability that the error signal $\text{err}_t$ in \eqref{eq: noisy err_t non Gaussian} correctly reports an error, setting $\text{err}_t=1$, is lower bounded by $1-F^+(\omega)$. Therefore, the probability that the true violation rate \text{violation-rate($T$)} no larger than the estimated violation rate $\sum_{t=1}^{T}\text{err}_t/T=1$ is lower  bounded by the probability that all the errors correctly reported. This is, in turn, lower bounded by $(1-F^+(\omega))^T$ by the independence of the observation noise variables $\{\epsilon_{q,t}\}_{t=1}^T$.
\end{proof}

As specified in Algorithm~\ref{table: p-safe-bocp}, \textsf{P-S{\footnotesize AFE}-B{\footnotesize OCP}} follows the same steps in \textsf{D-S{\footnotesize AFE}-B{\footnotesize OCP}} with the caveat that the error signal \eqref{eq: noisy err_t non Gaussian} is used in lieu of \eqref{eq: d-err_t}. As we will prove next, the correction applied via the safety error signal \eqref{eq: noisy err_t non Gaussian} is sufficient to meet the probabilistic safety requirement \eqref{eq: probabilistic goal}.

\subsection{Safety Guarantees}

The safety guarantee of \textsf{P-S{\footnotesize AFE}-B{\footnotesize OCP}} is summarized in the following theorem. 

\begin{theorem}[Safety Guarantee of \textsf{P-S{\footnotesize AFE}-B{\footnotesize OCP}}]\label{theorem: p-safe-bocp}
    Under noisy observations of the constraint function and Assumption \ref{ap: assump 1}, \textsf{P-S{\footnotesize AFE}-B{\footnotesize OCP}} satisfies the probabilistic safety constraint \eqref{eq: probabilistic goal} for any pre-determined target violation rate $\alpha \in (0,1]$ and target reliability level $\delta \in (0,1)$.
\end{theorem}

\begin{proof} Using the same arguments as in the proof of Theorem 2, the estimated violation rate can be upper bounded with probability 1 as \begin{align}
    \frac{1}{T}\sum_{t=1}^T\text{err}_t \leq \frac{1+\eta(1-\alpha_\text{algo})-\Delta \alpha_1}{T\eta}  +\alpha_\text{algo}.
\end{align} Using this bound with  Lemma~\ref{lemma: type ii error}, we conclude that, with probability at least $1-\delta$, in which the probability is taken over the observation noise variables $\{\epsilon_{q,t}\}_{t=1}^T$, we have bound on the \emph{true violation rate}
\begin{align}
    \text{violation-rate($T$)}\leq\frac{1}{T}\sum_{t=1}^T\text{err}_t \leq \alpha,
\end{align}
which recovers the probabilistic safety constraint \eqref{eq: probabilistic goal}.
\end{proof}

\subsection{Data-Driven Probability Bound}\label{ssec: est upper bound}
A possible challenge in applying \textsf{P-S{\footnotesize AFE}-B{\footnotesize OCP}} in practice is the fact that an upper bound $F^+(\omega)$ on the probability $\text{Pr}(\epsilon_{q,t}\geq\omega)$ may not be known \emph{a priori}. In this subsection, we provide a data-driven approach for evaluating an upper bound on the probability $\text{Pr}(\epsilon_{q,t}\geq\omega)$, assuming only access to independent and identically distributed (i.i.d.) observation noise samples.

\begin{lemma}
    [Estimated Upper Bound]\label{remark: est upper}
    Assume access to i.i.d. observation noise samples $\{\epsilon_{q,i}\}_{i=1}^m$. The empirical estimate of the right-tail probability 
    \begin{align}
        \hat{F}^+(\omega)=\frac{1}{m}\sum_{i=1}^m\mathbbm{1}(\epsilon_{q,i}>\omega),\label{eq: upper bound estimation}
    \end{align}
    when offset by $\psi>0$, provides an upper bound on $\text{Pr}(\epsilon_{q,i}>\omega)$ with probability
    \begin{align}
        &\Pr\big(\Pr(\epsilon_{q,i} \geq \omega, \forall \omega\in\mathbbm{R})\leq\hat{F}^+(\omega) + \psi \big) \nonumber\\&\geq 1-\exp(-2m\psi^2)
    \end{align}
    for any $\psi>\sqrt{\ln 2/2m}$.
\end{lemma}
Lemma~\ref{remark: est upper} is a direct application of Dvoretsky-Kiefer-Wolfowitz inequality \cite{massart1990tight}. 

Consequently, by using $\hat{F}^+(\omega) + \psi$ in lieu of $F^+(\omega)$ in (\ref{eq: omega q}), we have the following modified safety guarantee of  \textsf{P-S{\footnotesize AFE}-B{\footnotesize OCP}}.

\begin{corollary} \label{corr}
    Under noisy observations of the constraint function, \textsf{P-S{\footnotesize AFE}-B{\footnotesize OCP}} with $\hat{F}^+(\omega) + \psi$, for any $\psi>0$, in lieu of $F^+(\omega)$ in \eqref{eq: omega q}  satisfies the guarantee
    \begin{align}
        \Pr(\text{violation-rate}(T)\leq \alpha)\geq (1-\exp{(-2m\psi^2)})(1-\delta)
    \end{align}
    for any pre-determined target violation rate $\alpha \in (0,1]$ and target reliability level $\delta \in (0,1)$, where the probability is taken with respect to the observation noise variables $\{\epsilon_{q,t}\}_{t=1}^T$ as well as the $m$ i.i.d. noise samples in (\ref{eq: upper bound estimation}).
\end{corollary}
Corollary~\ref{corr} is obtained by combining Lemma~\ref{remark: est upper} and Theorem~\ref{theorem: p-safe-bocp}. Intuitively, with an increasing number $m$ of the constraint observation noise samples, the tightness of the safety guarantee in Theorem \ref{theorem: p-safe-bocp} is enhanced as a result of increasingly accurate observation noise estimation.

\section{Numerical Results For a Synthetic Benchmark}\label{sec: experiments}
In this section, we detail experimental results aimed at comparing \textsf{S{\footnotesize AFE}-B{\footnotesize OCP}} with \textsf{S{\footnotesize AFE}O{\footnotesize PT}} \cite{Felix2016safeopt} on a synthetic benchmark inspired by \cite{Felix2016safeopt}.

\subsection{Synthetic Dataset}\label{ssec: synthetic data}
In a manner similar to \cite{Felix2016safeopt}, we focus on a synthetic setting with a scalar optimization variable $\mathbf{x}\in\mathbbm{R}$ in which the objective function $f(\mathbf{x})$ is a realization of a GP with zero mean and RBF kernel  $\kappa^*(\mathbf{x},\mathbf{x}')$ \eqref{eq: rbf kernel} with bandwidth $h^*=1/1.62$, while the constraint function $q(\mathbf{x})$ is a function in this RKHS $\mathcal{H}_{\kappa^*}$ which has the form \eqref{eq: rkhs for q} with coefficients $\{a_i\}_{i=0}^{10}=[-0.05,-0.1,0.3,-0.3,0.5,0.5,-0.3,0.3,-0.1,-0.05]$ and scalars $\{\mathbf{x}_i\}_{i=1}^{10}=[-9.6,-7.4,-5.5,-3.3,-1.1,1.1,3.3,5.5,\\7.4,9.6]$. Accordingly, the constraint function $q(\mathbf{x})$ has RKHS norm $||q||_{\kappa^*}=1.69$ in  \eqref{eq: norm bound B}. In order to investigate the impact of misspecification of GP (see Sec.~\ref{ssec: rkhs}) on Safe-BO including the proposed \textsf{S{\footnotesize AFE}-B{\footnotesize OCP}}, we consider the two cases: (\emph{i}) \emph{well-specified GP} that uses $\kappa^*(\mathbf{x},\mathbf{x}')$ for the GP kernel, i.e., $\kappa(\mathbf{x},\mathbf{x}')=\kappa^*(\mathbf{x},\mathbf{x}')$; (\emph{ii}) \emph{misspecified GP} that uses RBF kernel with smaller bandwidth $h=1/14.58 < h^*$, i.e., $\kappa(\mathbf{x},\mathbf{x}')\neq \kappa^*(\mathbf{x},\mathbf{x}')$, with unknown $||q||_{\kappa}$.

\begin{figure}[t]

  \centering
  \centerline{\includegraphics[scale=0.46]{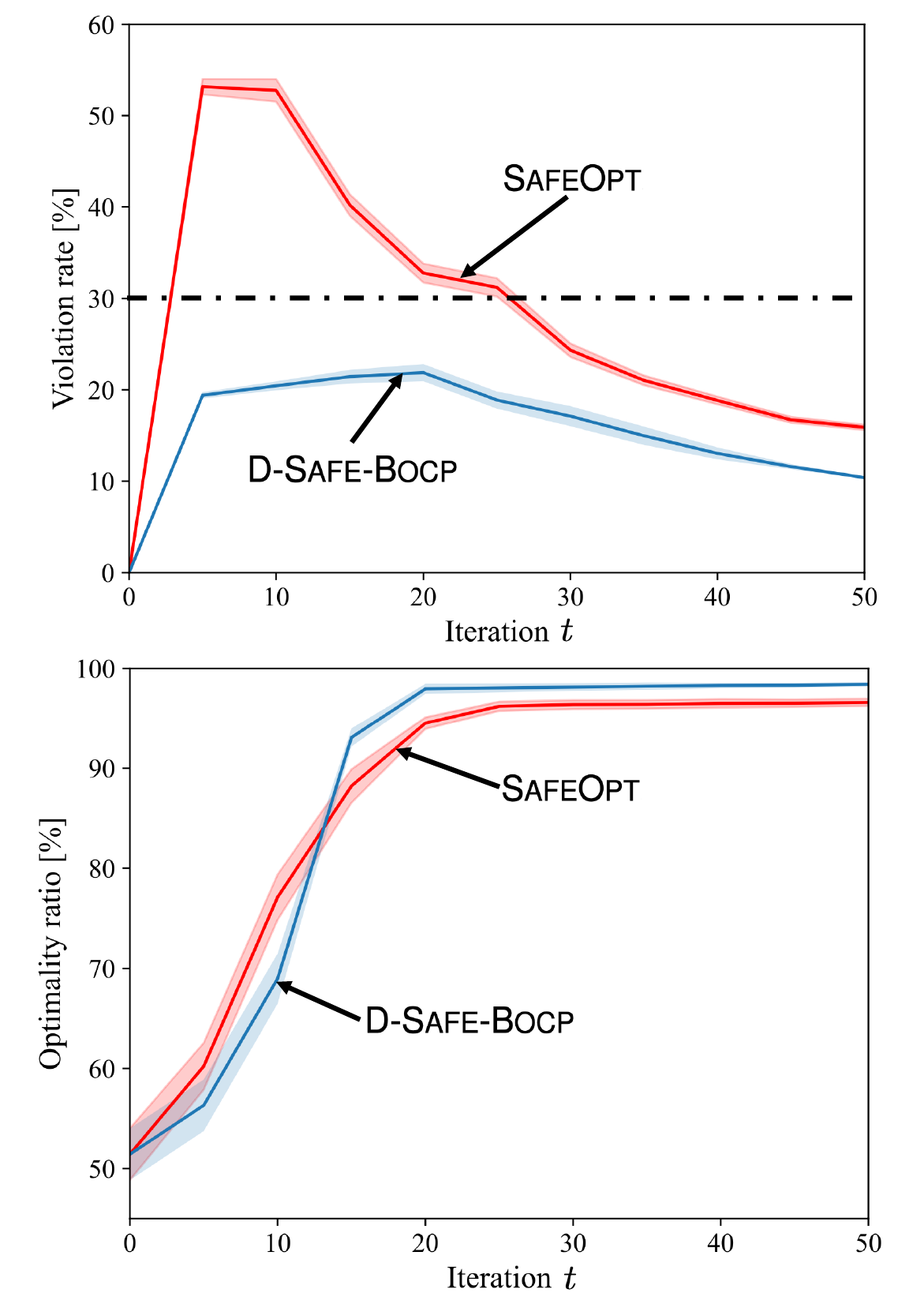}}
  \caption{Violation-rate$(t)$ (top) and optimality ratio (bottom) against iteration index $t$ with target violation rate $\alpha=0.3$ (dot-dashed line), update rate $\eta=2$, misspecified kernel bandwidth $h=1/14.58$, RKHS norm bound $B=||q||_{\kappa^*}$ and total number of iteration $T=50$.}\label{fig: vio_t}

\end{figure}


As discussed throughout the paper, the scaling parameter for the constraint function $q(\mathbf{x})$ in \eqref{eq: q credible interval} is \emph{a priori} determined by \eqref{eq: rkhs beta} for \textsf{S{\footnotesize AFE}O{\footnotesize PT}}, and is \emph{adapted} by feedback via $\beta_{t+1}=\varphi(\Delta \alpha_{t+1})$  \eqref{eq: update rule} for the proposed \textsf{S{\footnotesize AFE}-B{\footnotesize OCP}}, while we fix the scaling parameter for the objective function $f(\mathbf{x})$ in \eqref{eq: f credible interval} to $3$ since it does not affect the safety guarantee for both \textsf{S{\footnotesize AFE}O{\footnotesize PT}} (see \cite[Theorem 6]{srinivas2012tit}) and \textsf{S{\footnotesize AFE}-B{\footnotesize OCP}}. The objective observation noise variance is set to $\sigma_f^2=2.5\times10^{-3}$; and the initial safe decision is chosen as $\mathbf{x}_0=0$ for which we have $q(\mathbf{x}_0)=0.946>0$. For \textsf{S{\footnotesize AFE}-B{\footnotesize OCP}}, we set the update rate in \eqref{eq: update rule}  to $\eta=2.0$.  All results are averaged over 1,000 experiments, with error bars shown to encompass $95\%$ of the realizations. Each experiment corresponds to a random draw of the objective function and to random realization of the observation noise signals. The implementation examples of the synthetic benchmark can be found here\footnote{https://github.com/yunchuan-zhang/safe-bocp}.

\begin{figure}[t]

  \centering

  \includegraphics[scale=0.385]{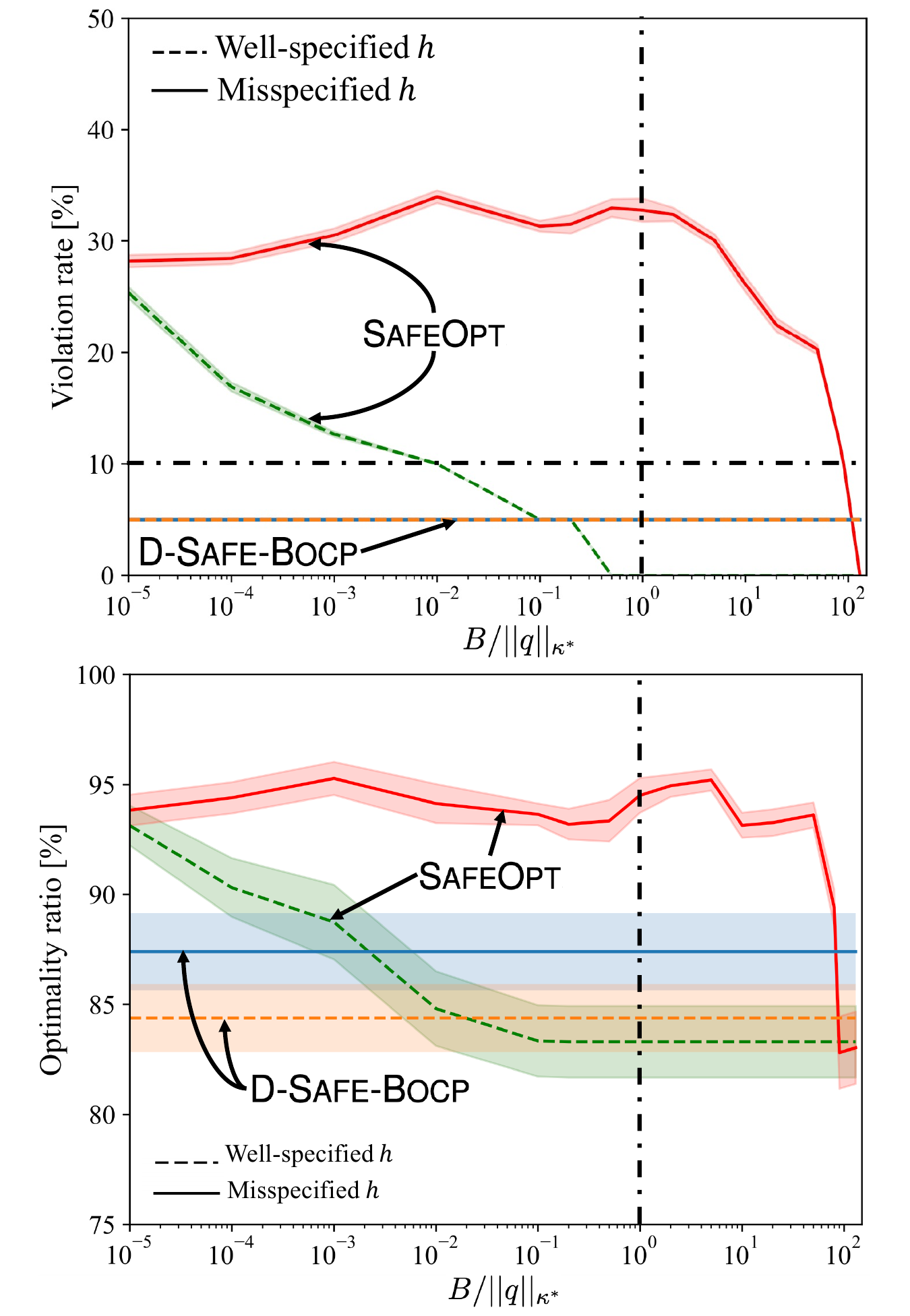}
  \caption{Violation rate \eqref{eq:goal} (top) and optimality ratio \eqref{eq: optimality ratio} (bottom) against the ratio between the  RKHS norm bound $B$ assumed by GP and the actual norm $||q||_{\kappa^*}$ in \eqref{eq: norm bound B}. The dashed lines are obtained with well-specified GP models,  which corresponds to kernel bandwidth $h=1/1.69$ (same one for $\kappa^*(\mathbf{x},\mathbf{x}')$), while the solid lines are obtained with misspecified GP models, having  kernel bandwidth $h=1/14.58$.
  \label{fig: vio_B}}
\end{figure}

\subsection{Deterministic Safety Requirement}\label{ssec: exp deterministic}
As explained in Sec. \ref{sec: safeopt}, \textsf{S{\footnotesize AFE}O{\footnotesize PT}} requires the GP model for the constraint function $q(\mathbf{x})$ to be well specified \eqref{eq: rkhs for q}--\eqref{eq: norm bound B} in order to meet safety conditions. To study the impact of violations of this assumption, we start by considering the noiseless case, i.e., $\sigma_q^2=0$, and we vary the kernel bandwidth $h$ adopted for the GP models used as surrogates for the objective and constraint functions as discussed earlier. 

Fig. \ref{fig: vio_t} shows the violation rate and optimality ratio against the iteration index $t$. For \textsf{D-S{\footnotesize AFE}-B{\footnotesize OCP}}, we set the update rate as $\eta=2$ and the target violation rate to $\alpha=0.3$, while \textsf{S{\footnotesize AFE}O{\footnotesize PT}} assumes target $\alpha=0$ with RKHS norm bound $B=||q||_{\kappa^*}$. For both schemes, the total number of iterations is $T=50$, and the misspecified GP with RBF kernel bandwidth $h=1/14.58$ is adopted.

The violation rate obtained by \textsf{S{\footnotesize AFE}O{\footnotesize PT}} is above the target $\alpha=0.3$ for a significant interval of time $t$, and it progressively falls below the target with a larger $t$, while \textsf{D-S{\footnotesize AFE}-B{\footnotesize OCP}} meets the deterministic safety requirement \eqref{eq:goal} with the pre-determined target $\alpha=0.3$ across all iterations. Furthermore, the optimality ratio obtained by \textsf{D-S{\footnotesize AFE}-B{\footnotesize OCP}} is larger than \textsf{S{\footnotesize AFE}O{\footnotesize PT}} after iteration $t=13$, converging to $97.5\%$ at iteration $t=20$. In contrast, \textsf{S{\footnotesize AFE}O{\footnotesize PT}} converges to optimality ratio of $94.5\%$ at iteration $t=25$, at which point the target safety level $\alpha=0.3$ is violated.


\begin{figure}[t]

  \centering
  \centerline{\includegraphics[scale=0.36]{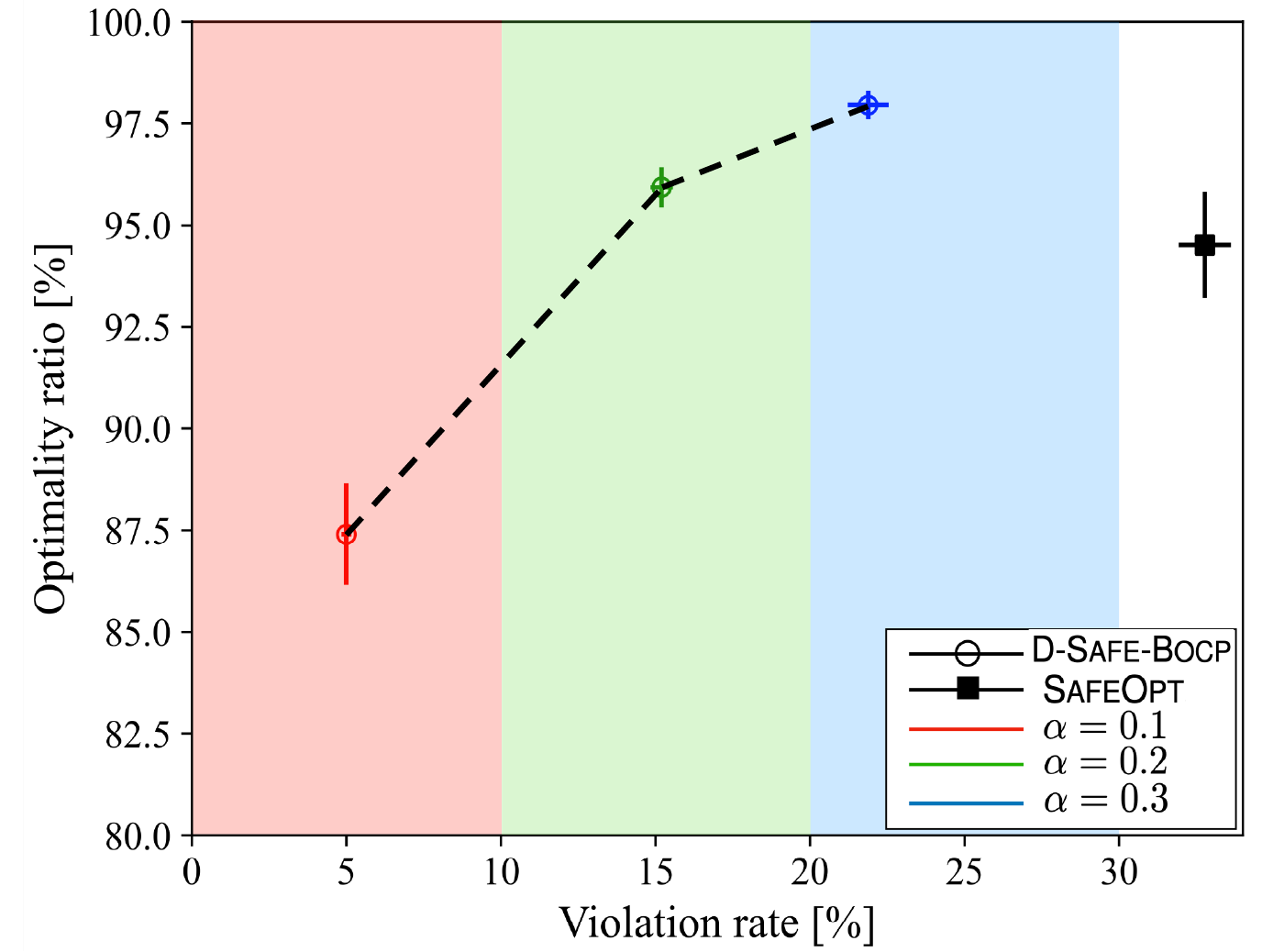}}
  \caption{$\text{Violation-rate}(T)$ and optimality ratio for different target violation rates $\alpha$ for \textsf{D-S{\footnotesize AFE}-B{\footnotesize OCP}}, with update rate $\eta=2$, misspecified kernel bandwidth $h=1/14.58$, and RKHS norm bound $B=||q||_{\kappa^*}$. The background colors represent intervals in which the safety requirement (\ref{eq:goal}) is met (see text for an explanation).}
  \label{fig: pv}

\end{figure}

Fig. \ref{fig: vio_B} shows the $\text{violation-rate}(T)$ in \eqref{eq:goal} with $T=20$, as well as the optimality ratio \eqref{eq: optimality ratio}, as a function of constant $B$ assumed by \textsf{S{\footnotesize AFE}O{\footnotesize PT}} for both well-specified and misspecified GPs, and the target violation rate is set to $\alpha=0.1$. Note that the performance of \textsf{D-S{\footnotesize AFE}-B{\footnotesize OCP}} does not depend on the value of $B$, which is an internal parameter for \textsf{S{\footnotesize AFE}O{\footnotesize PT}}, but it is affected by the choice of parameter $ h$. By Theorem \ref{theorem: safeopt theorem}, any value $B\geq||q||_{\kappa}$ in \eqref{eq: norm bound B} guarantees the safety of \textsf{S{\footnotesize AFE}O{\footnotesize PT}}. However, since RKHS norm for the misspecified GP is generally unknown, we plot violation rate and optimality ratio as functions of the ratio $B/||q||_{\kappa^*}$, to highlight the two regimes with well specified and misspecified value of $B$.  

Confirming Theorem \ref{theorem: safeopt theorem}, with a ratio $B/||q||_{\kappa^*}\geq1$ for the well-specified GP with kernel $\kappa(\mathbf{x},\mathbf{x}')=\kappa^*(\mathbf{x},\mathbf{x}')$, \textsf{S{\footnotesize AFE}O{\footnotesize PT}} is seen to strictly satisfy the deterministic safety constraint \eqref{eq:goal}, since the violation rate is equal to zero, as per its target. Instead, when $B/||q||_{\kappa^*}<1$, and/or when the GP is misspecified, i.e., $\kappa(\mathbf{x},\mathbf{x}')\neq\kappa^*(\mathbf{x},\mathbf{x}')$, the violation rate exceeds the target $\alpha$. In contrast, \textsf{D-S{\footnotesize AFE}-B{\footnotesize OCP}} obtains a violation rate below the target $\alpha$,  irrespective of kernel bandwidth $h$ assumed in GP.

In terms of optimality ratio, in the regime $B/||q||_{\kappa^*}\geq1$, with a well-specified GP parameter $h$, \textsf{S{\footnotesize AFE}O{\footnotesize PT}} achieves around 83\%, while \textsf{D-S{\footnotesize AFE}-B{\footnotesize OCP}} obtains the larger optimality ratio 84.5\%. In contrast, with a misspecified value $h$, \textsf{D-S{\footnotesize AFE}-B{\footnotesize OCP}} achieves an optimality ratio around 87.5\%, while the optimality ratio of \textsf{S{\footnotesize AFE}O{\footnotesize PT}} is larger, but this comes at the cost of the violation of the safety requirement. Note that a misspecified value of the kernel bandwidth $h$ does not necessarily reduce the performance of \textsf{D-S{\footnotesize AFE}-B{\footnotesize OCP}}, which is improved in this example.

The trade-off between violation rate and optimality ratio is studied in Fig. \ref{fig: pv} by varying the target violation rate $\alpha$ for \textsf{D-S{\footnotesize AFE}-B{\footnotesize OCP}}. For each value of $\alpha$,  we show the achieved pair of violation rate and optimality ratio, along with the corresponding realization ranges along the two axes. Recall that for \textsf{S{\footnotesize AFE}O{\footnotesize PT}} the assumed target is $\alpha=0$, and hence one pair is displayed. We focus here on the misspecified GP case, i.e., $\kappa(\mathbf{x},\mathbf{x}')\neq \kappa^*(\mathbf{x},\mathbf{x}')$, while the  \textsf{S{\footnotesize AFE}O{\footnotesize PT}} parameter $B$ is selected to the ``safe'' value $B=||q||_{\kappa^*}$, which is unaware of kernel misspecification.

\begin{figure}[t]

  \centering
  \centerline{\includegraphics[scale=0.845]{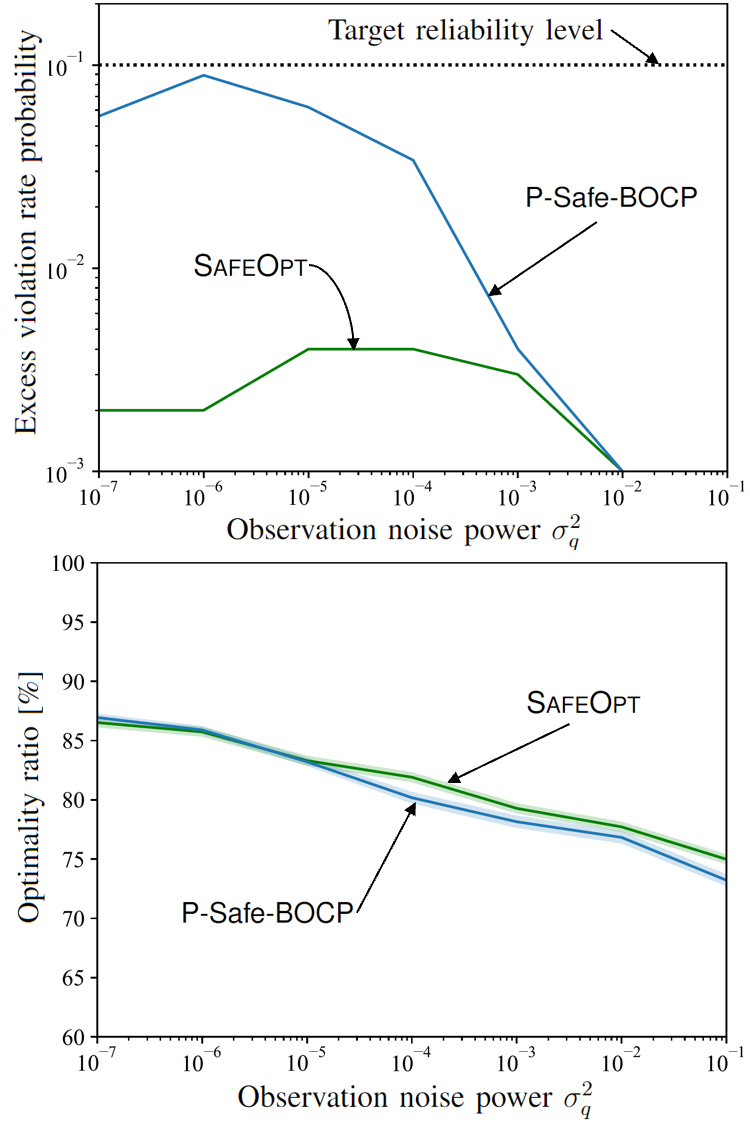}}
  \caption{Probability of excessive violation rate \eqref{eq: probabilistic goal} (top) and optimality ratio \eqref{eq: optimality ratio} (bottom) as a function of constraint observation noise power $\sigma^2_q$, with update rate $\eta=2$, RKHS norm bound $B=10||q||_{\kappa^*}$, and well-specified kernel bandwidth $h=1/1.62$. }
  \label{fig: well noise}

\end{figure}

For each value of $\alpha\in \{0.1,0.2,0.3\}$, the figure highlights the intervals of violation rates that meet the safety requirement (\ref{eq:goal}) using different colors. Specifically, for $\alpha=0.1$, all violation rates below $0.1$ are acceptable, as denoted by the red interval; for $\alpha=0.2$, all violation rates in the red and green intervals are acceptable; and for $\alpha=0.3$, all violation rates below in the cyan, green, and red interval meet the safety constraint. 


The figure shows that the violation rate obtained by \textsf{S{\footnotesize AFE}O{\footnotesize PT}} exceeds its target $\alpha=0$, and thus the safety requirement is violated. In contrast, as per the theory developed in this paper, \textsf{D-S{\footnotesize AFE}-B{\footnotesize OCP}} meets violation-rate requirement for all values of the target $\alpha$. Moreover, as the tolerated violation rate $\alpha$ increases, the optimality ratio of \textsf{D-S{\footnotesize AFE}-B{\footnotesize OCP}} is enhanced, indicating a trade-off between the two metrics. When increasing the target violation rate $\alpha$, \textsf{D-S{\footnotesize AFE}-B{\footnotesize OCP}} raises the algorithmic target level $\alpha_{\text{algo}}$ in \eqref{eq:alpha_algo}, making it possible for the optimizer to reduce the time spent under the maximally cautious scaling $\beta_t=\infty$ in \eqref{eq:betavar}. Consequently, with a larger $\alpha$, the optimality ratio of \textsf{D-S{\footnotesize AFE}-B{\footnotesize OCP}} gains from more explorations of the objective function $f(\mathbf{x})$.




\begin{figure}[t]

  \centering
  \centerline{\includegraphics[scale=0.83]{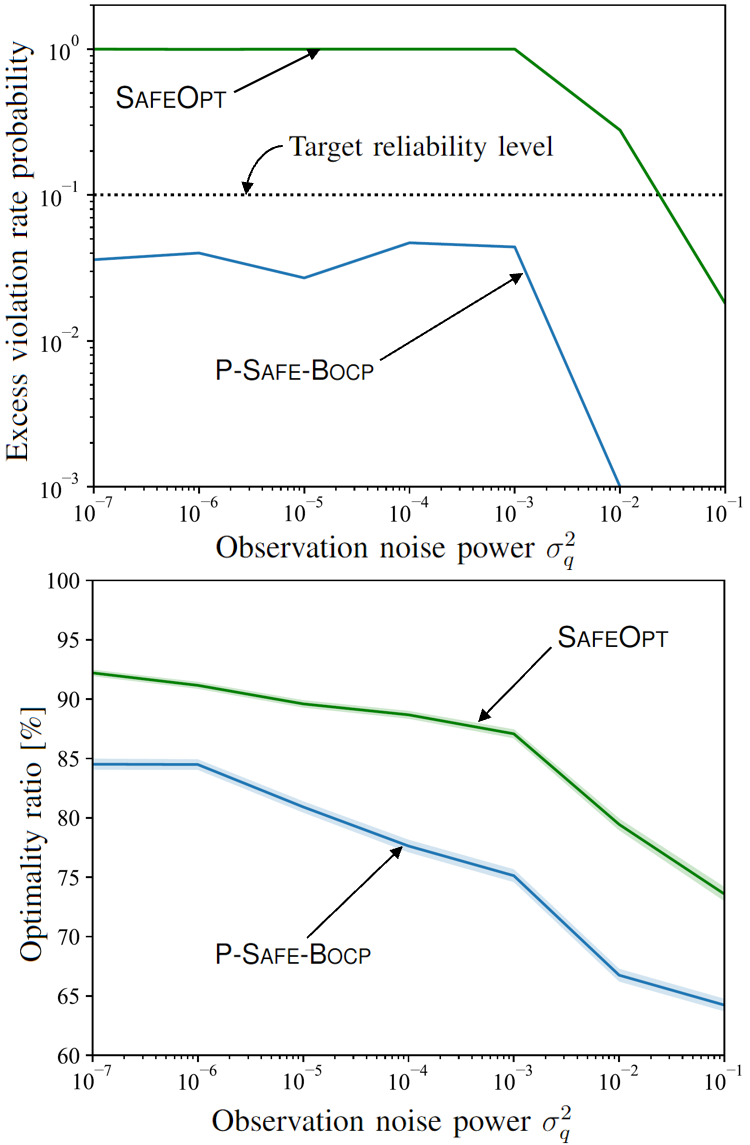}}
  \caption{Excess violation rate probability \eqref{eq: probabilistic goal} (top) and optimality ratio \eqref{eq: optimality ratio} (bottom) as a function of constraint observation noise power $\sigma^2_q$, with update rate $\eta=2$, RKHS norm bound $B=10||q||_{\kappa^*}$, and misspecified kernel bandwidth $h=1/14.58$. }
  \label{fig: mis noise}

\end{figure}

\subsection{Probabilistic Safety Constraint}\label{ssec: synthetic p-safe-bocp}

We now turn to considering scenarios with Gaussian observation noise $\sigma^2_q>0$, and aim at evaluating the performance in terms of  probabilistic safety requirement \eqref{eq: probabilistic goal} and optimality ratio (\ref{eq: optimality ratio}). We set the total number of iterations $T=25$, the target reliability level $1-\delta=0.9$ with target violation 
rate $\alpha=0.1$ for \textsf{P-S{\footnotesize AFE}-B{\footnotesize OCP}}, and with $\alpha=0$ for $\textsf{S{\footnotesize AFE}O{\footnotesize PT}}$ in accordance with $\textsf{S{\footnotesize AFE}O{\footnotesize PT}}$'s design. For the latter scheme, we set the ``safe'' value $B=10||q||_{\kappa^*}$, while we consider both the well-specified kernel bandwidth $h=h^*=1/1.62$, and the misspecified one $h=1/14.58<h^*$,  as considered also in the previous set of experiments. For all schemes, the excess violation rate probability in \eqref{eq: probabilistic goal} is obtained by averaging over 10,000 realizations. 



We plot the excess violation rate probability \eqref{eq: probabilistic goal} and the optimality ratio in Fig. \ref{fig: well noise}  and Fig. \ref{fig: mis noise} against the observation noise power $\sigma^2_q$. The first figure corresponds to the case of a well-specified kernel bandwidth while for the second we adopted misspecified value. Confirming the theory, in the former case, both \textsf{S{\footnotesize AFE}O{\footnotesize PT}} and \textsf{P-S{\footnotesize AFE}-B{\footnotesize OCP}} attain an excess violation rate probability below the target level $1-\delta$. In contrast, for a misspecified kernel, \textsf{S{\footnotesize AFE}O{\footnotesize PT}} can only satisfy the constraint \eqref{eq: probabilistic goal} for sufficiently large observation noise, but \textsf{P-S{\footnotesize AFE}-B{\footnotesize OCP}} still meets the probability safety constraint \eqref{eq: probabilistic goal}.  We note that a larger observation noise is beneficial to \textsf{S{\footnotesize AFE}O{\footnotesize PT}} in terms of safety since it forces a larger level of pessimism in the definition of the safe set $\mathcal{S}_{t+1}$ in \eqref{eq: safebo safe set}. 

In terms of optimality ratio, larger observation noise power $\sigma^2_q$ generally yields a degraded optimality ratio. In the well-specified regime considered in Fig. \ref{fig: well noise}, both schemes have comparable performance and the optimality ratio gap is no more than 5\%. In the misspecified regime demonstrated in Fig. \ref{fig: mis noise}, the performance levels are not comparable, since the gains of recorded for \textsf{S{\footnotesize AFE}O{\footnotesize PT}} come at the cost of violations of the safety constraint \eqref{eq: probabilistic goal}, except for a sufficiently large observation noise power, here $\sigma^2_q\geq0.1$. 

\section{Numerical Results for Real World Applications}\label{sec: real world}
In this section, we compare \textsf{S{\footnotesize AFE}O{\footnotesize PT}} \cite{sui2015safeopt} and \textsf{S{\footnotesize AFE}-B{\footnotesize OCP}} in two real-world applications, with the goal of validating the safety gains obtained by the proposed method along the optimization process.

\subsection{Safe Movie Recommendation}\label{ssec: safe movie}

As in \cite{sui2015safeopt}, consider a system that sequentially recommends movies to a user. Each user assigns a score from 1 to 5 to a recommended movie. Following standard matrix factorization algorithms, we introduce a feature vector $\mathbf{x} \in \mathbbm{R}^d$ for each movie. Accordingly, selecting a movie amounts to choosing a vector $\mathbf{x}$ within a set of possible movies. Denote as $r(\mathbf{x})$ the rating assigned by a user to movie $\mathbf{x}$. A recommendation is deemed to be unsafe if the user assigns it a rating strictly smaller than 4, i.e., if $r(\mathbf{x})< 4$. Accordingly, we set both objective function $f(\mathbf{x})$ and constraint function $q(\mathbf{x})$ to be equal to $f(\mathbf{x})=
q(\mathbf{x})=r(\mathbf{x})-4$. We focus on the deterministic safety constraint (\ref{eq:goal}), since the ratings are assumed to be observed with no noise.

To define a GP model for the function that maps a movie feature vector $\mathbf{x}$ to a rating $r(\mathbf{x})$, we need to specify a kernel function, which describes the similarity between movies. As in \cite{sui2015safeopt}, we adopt the linear kernel\begin{align}
\kappa(\mathbf{x},\mathbf{x}')=\mathbf{x}^{\sf T} \mathbf{x}', \label{eq: linear kernel}
\end{align} for any two movie feature vectors $\mathbf{x}$ and $\mathbf{x}'$.

The feature vectors $\mathbf{x}$ for movies are optimized using the MovieLens-100k dataset \cite{harper2016movielens}, which includes sparse rating observations of 1,680 movies from 943 users. Specifically, as in \cite{sui2015safeopt}, we randomly select 200 users to form the training data set, and we set $d=20$ for the size of the feature vectors. Training applies the standard matrix factorization algorithm \cite{lee2000algorithms}. For testing,  we pick the 10 test users, not selected for training, that have the most rated movies, and remove the movies with no rating from the possible selections. 



\begin{figure}[t]

  \centering
  \centerline{\includegraphics[scale=0.42]{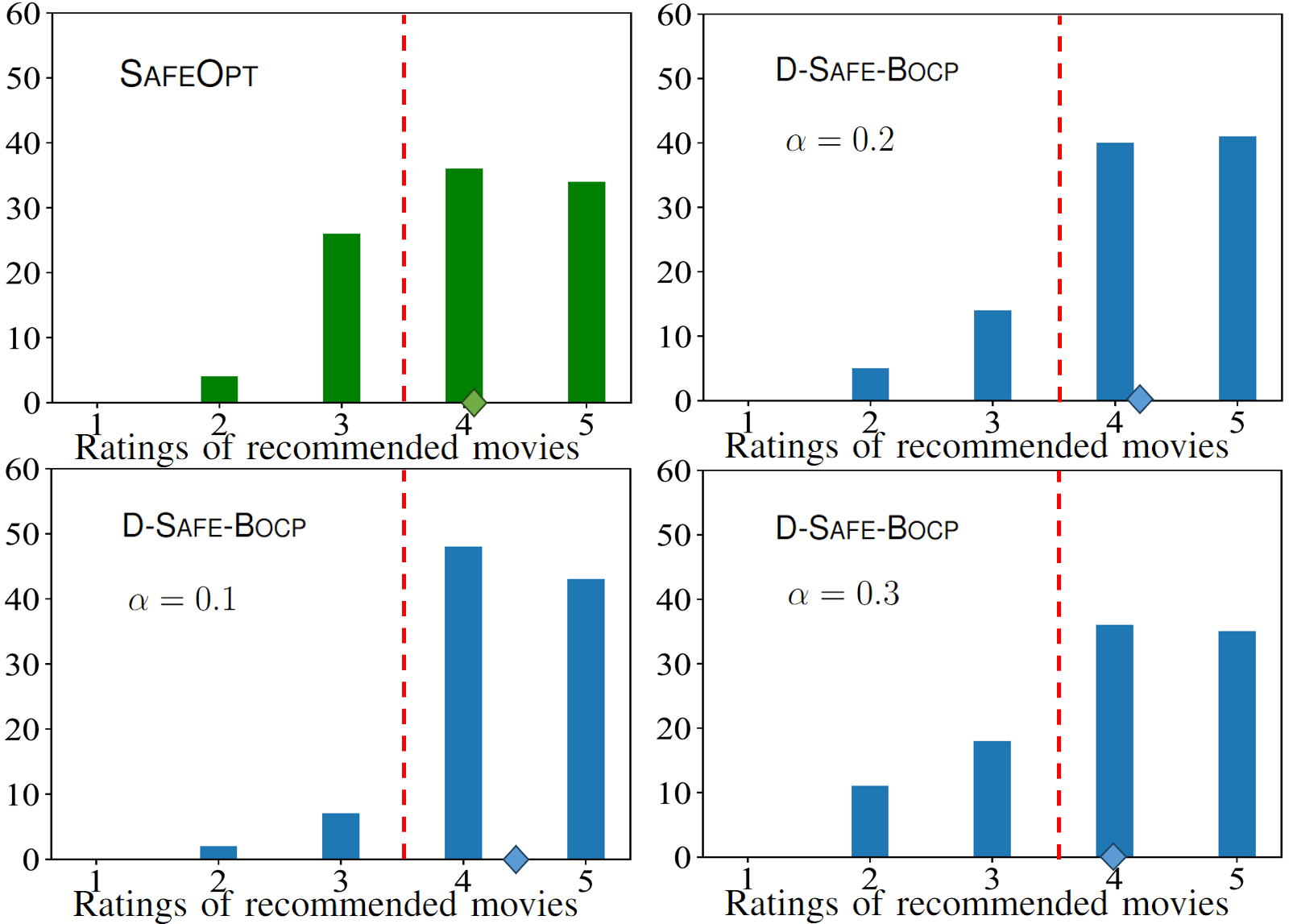}}
  \caption{Histograms of the ratings of recommended movies by  \textsf{S{\footnotesize AFE}O{\footnotesize PT}}, as well by \textsf{D-S{\footnotesize AFE}-B{\footnotesize OCP}}  under different target violation rates $\alpha$. The dashed lines represent the safety threshold for the recommendations, and the marker on the horizontal axis represents the average rating of the  recommendations.}
  \label{fig: safe movie hist}

\end{figure}




Since the true underlying function that maps movie feature vector $\mathbf{x}$ to rating $r(\mathbf{x})$ is unknown, it is not possible to evaluate the RKHS norm $||q||_{\kappa}$ in \eqref{eq:RKHS_norm} required by \textsf{S{\footnotesize AFE}O{\footnotesize PT}}. Accordingly, as in \cite{Felix2016safeopt}, we set $B=3$ a priori for \textsf{S{\footnotesize AFE}O{\footnotesize PT}}. In this experiment, we run both \textsf{S{\footnotesize AFE}O{\footnotesize PT}} and \textsf{D-S{\footnotesize AFE}-B{\footnotesize OCP}} for $T=100$ iterations on the selected 10 test users.  We randomly select a movie rated as 4 for each test user as the initial starting point $\mathbf{x}_0$, and set the update rate $\eta=10$ for \textsf{D-S{\footnotesize AFE}-B{\footnotesize OCP}}. 



To evaluate the performance of both schemes, we show in Fig. \ref{fig: safe movie hist}  the histograms of the ratings across all selected movies during the optimization procedure. The vertical dashed line represents the safety threshold between safe and unsafe recommendations. The marker on the horizontal axis marks the average rating. For \textsf{D-S{\footnotesize AFE}-B{\footnotesize OCP}} we have the flexibility to vary the target violation rate $\alpha$, while we recall that for \textsf{S{\footnotesize AFE}O{\footnotesize PT}} the target is $\alpha=0$. 

The top-left panel of Fig. \ref{fig: safe movie hist} shows that \textsf{S{\footnotesize AFE}O{\footnotesize PT}} does not meet the safety requirement \eqref{eq:goal} with $\alpha=0$ owing to the mismatch between the assumptions made by the scheme and the true, unknown, constraint function. The remaining panels demonstrate that, in contrast,  \textsf{D-S{\footnotesize AFE}-B{\footnotesize OCP}} can correctly control the fraction $\alpha$ of unsafe recommendations. 

\subsection{Chemical Reaction Optimization}\label{ssec: PFR}
Finally, we consider the plug flow reactor (PFR) problem introduced in \cite{kang2019glucose}\footnote{Simulator available at https://github.com/VlachosGroup/Fructose-HMF-Model}, which seeks for optimal chemical reaction parameters $\mathbf{x}\in [140,200]\times (0,1] \subset \mathbbm{R}^2$, with the first dimension being the temperature ($^\circ C$) and the second being  the pH value. The goal is to maximize the yield (\%), which we set as the objective $f(\mathbf{x})$, while keeping an acceptable selectivity level (\%), which we denote as $s(\mathbf{x})$. We refer to \cite{kang2019glucose} for a precise definition of these terms.

\begin{figure}[t]

  \centering
  \centerline{\includegraphics[scale=0.83]{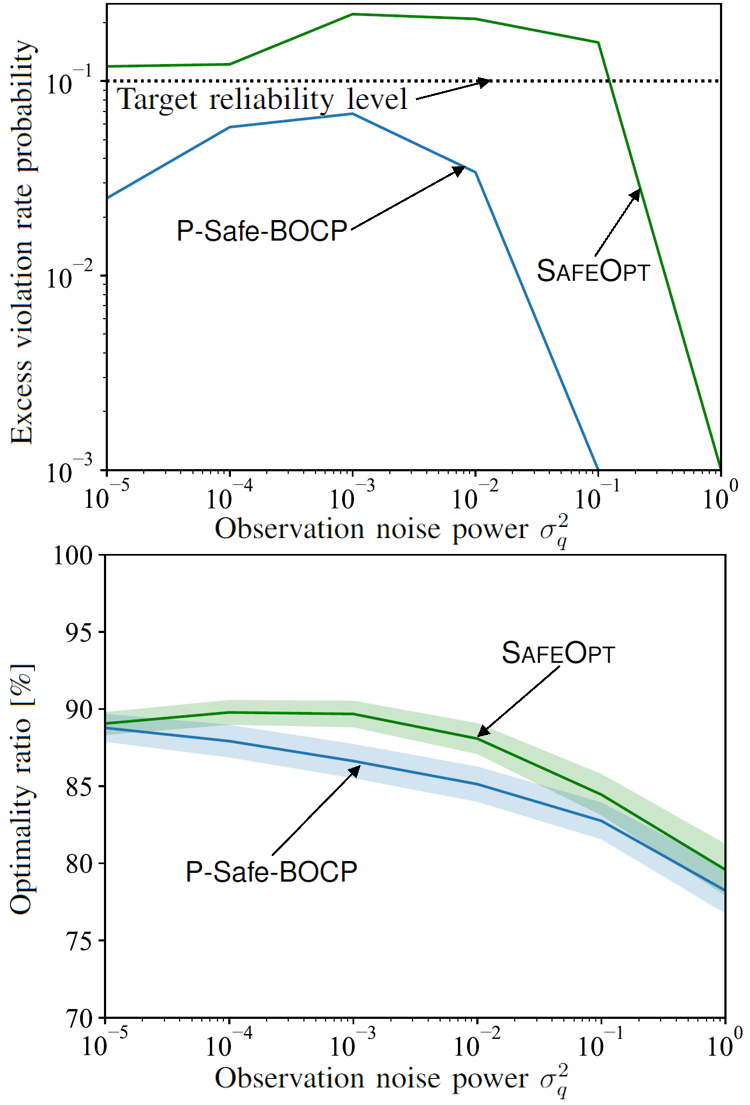}}
  \caption{Probability of excessive violation rate \eqref{eq: probabilistic goal} (top) and optimality ratio \eqref{eq: optimality ratio} (bottom) as a function of constraint observation noise power $\sigma^2_q$, with update rate $\eta=2$, RKHS norm bound $B=3$, and kernel bandwidth $h=1/2.88$ for the chemical reaction problem. }
  \label{fig: pfr}

\end{figure}


A reaction vector is deemed to be unsafe if the resulting selectivity level is lower than the corresponding yield, hence we define the constraint function as $q(\mathbf{x})=s(\mathbf{x})-f(\mathbf{x})$.  We assume the presence of non-zero Gaussian observation noise $z_t$ for the constraint function, i.e., $\sigma^2_q>0$. Accordingly, we focus on the probabilistic safety constraint \eqref{eq: probabilistic goal}, and compare the performance of \textsf{S{\footnotesize AFE}O{\footnotesize PT}} and \textsf{P-S{\footnotesize AFE}-B{\footnotesize OCP}}. We adopt GP surrogates model for both $f(\mathbf{x})$ and $q(\mathbf{x})$ with RBF kernel having bandwidth $h=1/2.88$.


Similar to Sec. \ref{ssec: safe movie}, since the smoothness property of the true underlying functions $q(\mathbf{x})$ is unknown, we assume the constant $B=3$ for \textsf{S{\footnotesize AFE}O{\footnotesize PT}} \cite{Felix2016safeopt}. The initial decision $\mathbf{x}_0$ is randomly chosen among the a priori known safe decisions that satisfy the constraint  $q(\mathbf{x_0})\geq 0$, and we  set the total number of optimization round to be $T=50$. Other settings are as in Sec. \ref{ssec: synthetic data}.  

In a similar manner to Sec. \ref{ssec: synthetic p-safe-bocp}, we demonstrate the excess violation rate probability \eqref{eq: probabilistic goal} and the optimality ratio in Fig. \ref{fig: pfr} as a function of the observation noise power $\sigma^2_q$. Confirming the discussion in Sec. \ref{ssec: synthetic p-safe-bocp} and the theory, \textsf{P-S{\footnotesize AFE}-B{\footnotesize OCP}} is seen to meet the probabilistic safety constraint \eqref{eq: probabilistic goal} irrespective of observation noise power, while \textsf{S{\footnotesize AFE}O{\footnotesize PT}} can only attain an excess violation rate probability below the target $1-\delta$ when the observation noise power is sufficiently large.


\section{Conclusions}\label{sec: conclusion}
In this work, we have introduced \textsf{S{\footnotesize AFE}-B{\footnotesize OCP}}, a novel BO-based zero-th order sequential optimizer that provably guarantees safety requirements irrespective of the properties of the constraint function. The key mechanism underlying \textsf{S{\footnotesize AFE}-B{\footnotesize OCP}} adapts the level of pessimism adopted during the exploration of the search space on the basis of noisy safety feedback received by the system. From synthetic experiment to real-world applications, we have demonstrated that the proposed \textsf{S{\footnotesize AFE}-B{\footnotesize OCP}} performs competitively with state-of-the-art schemes in terms of optimality ratio, while providing for the first time assumption-free safety guarantees.

Although in this work we have built on \textsf{S{\footnotesize AFE}O{\footnotesize PT}} for the acquisition process, the proposed framework could be generalized directly to any other Safe-BO schemes, such as  \textsf{G{\footnotesize O}OSE}\cite{ matteo2019goose}. Other possible extensions include accounting for multiple constraints, multi-fidelity approximations on objective and constraints \cite{zhang2024multi}, as well as taking into account contextual information during the optimization process \cite{widmer2023tuning}.



\bibliographystyle{ieeetr}
\bibliography{refer}

\newpage

 



\vfill

\end{document}